%% file: main.tex
\documentclass{article} 
\usepackage{fullpage}  

\usepackage{times}  
\usepackage{helvet}  
\usepackage{courier}  
\usepackage[hyphens]{url}  
\usepackage{graphicx} 
\usepackage{bbm}

\usepackage{natbib}  
\usepackage{caption} 
\usepackage{algorithm}
\usepackage{algorithmic}
\usepackage{hyperref}
\usepackage{enumitem}

\usepackage{color-edits}

\addauthor{sw}{blue}
\addauthor{gv}{red}

\usepackage{newfloat}
\usepackage{listings}

\usepackage{booktabs}       
\usepackage{amsfonts}       
\usepackage{amsthm}
\usepackage{xcolor}
\usepackage{xspace}
\usepackage{amsmath}

\usepackage{mathtools}
\usepackage{authblk}

\newtheorem{theorem}{Theorem}[section]
\newtheorem{lemma}{Lemma}[section]
\newtheorem{remark}{Remark}[section]

\newtheorem{definition}{Definition}
\newtheorem{corollary}[theorem]{Corollary}

\usepackage[utf8]{inputenc}
\usepackage[capitalize,noabbrev]{cleveref}

\newif\ificml

\usepackage{soul}
\usepackage{multirow}

\title{Generating Private Synthetic Data with Genetic Algorithms}

\newcommand*\samethanks[1][\value{footnote}]{\footnotemark[#1]}
\author[1]{Terrance Liu\thanks{Equal contribution; all authors are listed in alphabetical order.}}
\author[2]{Jingwu Tang\samethanks}
\author[3]{Giuseppe Vietri\samethanks}
\author[1]{Zhiwei Steven Wu}
\affil[1]{Carnegie Mellon University}
\affil[2]{Peking University}
\affil[3]{University of Minnesota}

\input{math_help}
\date{}
\begin{document}
\maketitle

\begin{abstract}

We study the problem of efficiently generating differentially private synthetic data that approximate the statistical properties of an underlying sensitive dataset.  In recent years, there has been a growing line of work that approaches this problem using first-order optimization techniques. However, such techniques are restricted to optimizing differentiable objectives only, severely limiting the types of analyses that can be conducted. For example, first-order mechanisms have been primarily successful in approximating statistical queries only in the form of marginals for discrete data domains. In some cases, one can circumvent such issues by relaxing the task's objective to maintain differentiability. However, even when possible, these approaches impose a fundamental limitation in which modifications to the minimization problem become additional sources of error. Therefore, we propose \privGA{}, a private genetic algorithm based on {\em zeroth}-order optimization heuristics that do not require modifying the original objective. As a result, it avoids the aforementioned limitations of first-order optimization. We empirically evaluate \privGA{} against baseline algorithms on data derived from the American Community Survey across a variety of statistics---otherwise known as statistical queries---both for discrete and real-valued attributes. We show that \privGA{} outperforms the state-of-the-art methods on non-differential queries while matching accuracy in approximating differentiable ones. 
\end{abstract}

\input{Arxiv/docs/intro.tex}

\input{Arxiv/docs/prelims}

\input{Arxiv/docs/query_release_problem.tex}

\input{Arxiv/algorithms/gsd}
\input{Arxiv/docs/PrivGA_2}

\input{Arxiv/docs/experiments}

\input{Arxiv/docs/conclusion}

\section*{Acknowledgements}
We thank Gokul Swamy for his valuable input during our preliminary conversations, particularly his recommendation to explore genetic algorithms, proved instrumental in our research.  Additionally, we thank Shuai Tang, whose assistance  facilitated the successful execution of our experiments. TL and ZSW were supported in part by the NSF Award \#2120667 and a Cisco Research Grant.

\bibliographystyle{plainnat}
\bibliography{main}

\newpage
\appendix
\onecolumn
\input{Arxiv/docs/accuracy.tex}

\input{Arxiv/docs/algorithm_details.tex}
\input{Arxiv/docs/experiment_details.tex}

\end{document}

%% file: math_help.tex
\DeclareMathOperator{\argmin}{\text{arg}\min}

\newcommand{\rap}{\textsc{RAP}}

\newcommand{\rp}{\textsc{RP}}
\newcommand{\gn}{\textsc{GN}}
\newcommand{\rappp}{\textsc{RAP++}}
\newcommand{\pgm}{\textsc{PGM}}
\newcommand{\pgmem}{\textsc{PGM+EM}}
\newcommand{\gem}{\textsc{GEM}}
\newcommand{\smallDB}{\textsc{SmallDB}}

\newcommand{\mwem}{\textsc{MWEM}}
\newcommand{\dualquery}{\textsc{DualQuery}}
\newcommand{\fem}{\textsc{FEM}}

\newcommand{\Loss}{\text{L}}

\newcommand{\cX}{\mathcal{X}}

\newcommand{\cR}{\mathcal{R}}

\newcommand{\cN}{\mathcal{N}}

\newcommand{\cP}{\mathcal{P}}

\newcommand{\pp}[1]{\left( #1 \right)}

\newcommand{\dotp}[1]{\langle #1 \rangle}

\newcommand{\prob}{\text{Pr}}
\newcommand{\pr}[1]{\prob\left[ #1 \right]}

\newcommand{\Dhat}{\widehat{D}}
\newcommand{\Dbar}{\bar{D}}
\newcommand{\Dtil}{\widetilde{D}}

\newcommand{\simpleGA}{\textsc{SimpleGA}}
\newcommand{\GSD}{\textsc{Private-GSD}}
\newcommand{\privGA}{\GSD}

\newcommand{\RAP}{\textsc{RAP++}}

\newcommand{\acsreal}{\textsc{ACSreal}\xspace}
\newcommand{\acsmultitask}{\textsc{ACSmultitask}\xspace}
\newcommand{\acsemp}{\textsc{ACSemployment}\xspace}
\newcommand{\acscov}{\textsc{ACScoverage}\xspace}
\newcommand{\acsinc}{\textsc{ACSincome}\xspace}
\newcommand{\acsmob}{\textsc{ACSmobility}\xspace}
\newcommand{\acstra}{\textsc{ACStravel}\xspace}

\newcommand{\acsrealdata}[1]{\textsc{ACSreal-#1}\xspace}
\newcommand{\acsmobdata}[1]{\textsc{ACSmob-#1}\xspace}
\newcommand{\acsempdata}[1]{\textsc{ACSemp-#1}\xspace}
\newcommand{\acscovdata}[1]{\textsc{ACScov-#1}\xspace}
\newcommand{\acsincdata}[1]{\textsc{ACSinc-#1}\xspace}
\newcommand{\acstradata}[1]{\textsc{ACStra-#1}\xspace}

\newcommand{\invtemp}{{\sigma^{\text{it}}}}

\newcommand{\LR}{\textsc{lr}}

\def\gX{{\mathcal{X}}}

\newcommand{\anglebrack}[1]{\left\langle#1\right\rangle}

\newcommand{\norm}[1]{\left\lVert #1 \right\rVert}

\newcommand{\elite}{\mathcal{E}}

\newcommand{\Pmut}{P_{\text{mut}}}
\newcommand{\Pcross}{P_{\text{cross}}}

%% file: Arxiv/docs/intro.tex
\section{Introduction}

Access to high-quality data is critical for decision-making and data analysis. However, the reliance on sensitive data can reveal private information about the individuals in the data, such as medical conditions or financial status. Therefore, privacy concerns impose legal and ethical constraints on what one can access for data analysis. \emph{Differential privacy} \citep{dwork2006calibrating} has become an increasingly popular framework  for protecting sensitive information by providing a formal privacy guarantee that allows one to calibrate the trade-off between privacy and accuracy. Moreover, differentially private synthetic data has become especially attractive, given that such data can be accessed repeatedly without added privacy costs. Consequently, we focus on the problem of generating synthetic data that approximates various properties of the underlying sensitive dataset while providing privacy guarantees. 

A standard approach to this problem is to find a synthetic dataset that matches a large family of statistics derived from the underlying dataset. In this work, we focus on statistics in the form of  \textit{statistical queries}, which count the fraction of examples that satisfy some specific properties. This approach is often framed as synthetic data for \textit{private query release}, where the goal is to release answers to an extensive collection of statistical queries by outputting a synthetic dataset from which answers are derived from. There exists a long line of research studying synthetic data for private query release \citep{BLR08, MWEM, dualquery, HDMM, FEM, mckenna2019graphical, aydore2021differentially, liu2021leveraging, liu2021iterative}. Moreover, when the set of statistical properties is sufficiently rich, this approach has been shown to outperform differentially private deep generative models (e.g., GANs) in both capturing various statistical properties and enabling downstream machine learning tasks \citep{tao2021benchmarking, vietri2022private}. 

%
Even though the problem of generating synthetic data for private query release is shown to be computationally intractable in the worst case \citep{hardsynthetic, ullman2013answering},  there has been a line of work on practical algorithms that can perform well on real datasets \citep{dualquery, FEM, mckenna2019graphical, liu2021leveraging}. Before our work, the algorithms that achieve state-of-the-art performance all leverage gradient-based optimization \citep{liu2021iterative, vietri2022private} to minimize the error for certain classes of statistical queries. However, these methods require differentiability, severely limiting the types of statistical properties one can optimize. For example, existing algorithms such as \rap{} \citep{aydore2021differentially} and \gem{} \citep{liu2021iterative} focus specifically on using first-order mechanisms for approximating marginal queries that apply to discrete data only, requiring real-valued attributes be discretized first. 


Many natural classes of queries for data containing real-valued attributes are non-differentiable, such as prefixes (equivalently the CDF of an attribute), and their higher-dimensional extension, halfspaces. Moreover, these queries cannot be optimized directly by methods that operate over discrete data, even with the discretization of real-valued features. \citet{vietri2022private} circumvent discretization by approximating non-differentiable queries with differentiable surrogate functions. Thus, their approach, \rappp{}, can directly optimize over a large class of statistical queries. However, their method induces additional error due to the relaxation of the original minimization problem, which may not produce an optimal solution for the original one. \footnote{
In Appendix \ref{appx:rappp}, we show a simple example in which the continuous relaxation leads \rappp{} to a bad local minimum very far from the optimal solution---\emph{even} when absent of privacy constraints.}
\footnote{\cref{fig:badexample} gives a simple example where \rappp{} fails to optimize a set of $3$ prefix queries on $1$-dimensional real-valued data.}


In light of these challenges, we propose a new synthetic data generation algorithm, \privGA{}, that is capable of outputting \textit{mixed}-type data (i.e., containing both discrete and real-valued attributes) without requiring differentiability in its optimization objective or discretization of real-valued attributes. \GSD{} is a zeroth-order mechanism based on a genetic algorithm that avoids the limitations of first-order methods such as \rappp{}, which uses surrogate loss functions to ensure differentiability. Instead, \privGA{} can optimize the error objective directly. Inspired by the genetic algorithm $\simpleGA$ \citep{such2017deep}, our algorithm \GSD{} evolves a population of datasets over $G$ generations---starting from a population of $E$ randomly chosen datasets, where on each generation, the algorithm maintains a set of elite synthetic datasets. Then, the algorithm uses this elite set to produce the population for the next round by crossing samples (i.e., combining parameters) and applying specific mutations (i.e., perturbing parameters). For the purpose of private synthetic data, we devise new strategies for crossing and mutating samples. To illustrate the limitations of the first-order approach \rappp{} and how \GSD{} overcomes them, we provide a simple example in \cref{fig:badexample}, where both algorithms try to approximate a target distribution over real values.

\input{Figures/rap++_bad_example}

Moreover, given that our method only requires zeroth-order access to any objective function, we demonstrate that \privGA{} can directly optimize over a wide range of statistical queries that capture various statistical properties summarizing both discrete and real-valued data attributes. Specifically, we first conduct experiments in mixed-type data domains, evaluating on random (1) halfspace and (2) prefix queries. We then explore categorical-only data, in which we use (3) $k$-way categorical marginal queries. Lastly, we again evaluate \privGA{} on mixed-type data using (4) $k$-way \emph{binary-tree marginal}\footnote{See \cref{sec:experiments}  for definition of binary-tree marginals. } queries. Note that (3) and (4) can be represented as differentiable functions, while (1) and (2) cannot.

We summarize our contributions as the following:
\begin{itemize}
    \item We present \privGA{}, a versatile genetic algorithm that can generate synthetic data to approximate a wide range of statistical queries, regardless of their differentiability. To achieve this, we develop novel crossover and mutation functions that are crucial to the success of \privGA{} across various statistical queries. 
    \item We conduct an empirical evaluation on data derived from the American Community Survey. We note that our method is highly parallelizable, scaling to the high-dimensional data settings explored in our experiments.
    \item In the mixed-type data domain, we show that \privGA{} outperforms the state-of-the-art method, \rappp{}, on non-differentiable queries (i.e., halfspace and prefix queries). Furthermore, we show that \emph{even} for differentiable query classes (i.e., $k$-way categorical and binary-tree marginals), \privGA{} matches first-order optimization methods.\footnote{Our work is also the first to modify our baseline methods, \rap{} and \gem{}, to optimize over $k$-way binary-tree marginals for mixed type data, with the limitation being that real-valued attributes are discretized first.}
    \item We explore the utility of private synthetic data in the context of machine learning tasks. Our findings show that \privGA{} performs well compared to baseline methods, but also demonstrate that differentially private synthetic data for downstream ML tasks remains an open problem in the field.
\end{itemize}

\paragraph{Additional related work.}

In addition to algorithms \rap{}, \gem{}, and \rappp{} \citep{aydore2021differentially, liu2021iterative, vietri2022private}, another benchmark algorithm that uses first-order optimization is the PrivatePGM (\pgm{}) algorithm \citep{mckenna2019graphical} and its variations \citep{mckenna2022aim}.  Similar to \rap{} and \gem{}, \pgm{} does not directly does not support non-differentiable queries such as thresholds and half-spaces on real-valued data. While one can first discretize numerical features, \citet{vietri2022private} show that \rappp{} still performs better than \pgm{} on real-valued data for approximating prefix queries and downstream machine learning. Finally, there is a line of work \citep{XieLWWZ18, JordonYvdS18, gan2, gan3, vae, harder2021dp} that develops differentially private algorithms for training deep generative models (e.g., GANs, VAEs, etc.), but they generally have not shown promising results for enforcing consistency with simple classes of statistics when applied to tabular data \citep{tao2021benchmarking, vietri2022private}.

%% file: Figures/rap++_bad_example.tex
\begin{figure}
    \centering
    \ificml
\else
        \includegraphics[width=0.6\textwidth]{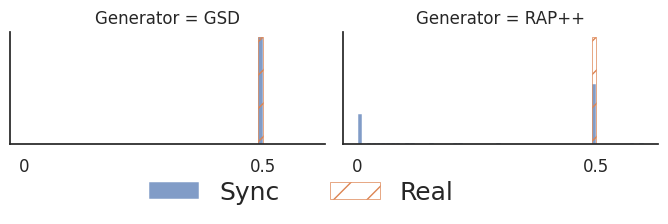}
    \fi
    \caption{
    The histogram visually demonstrates a specific scenario where the \rappp{} mechanism fails to generate synthetic data that approximates a collection of statistical prefix queries. In this particular scenario, both \GSD{} and \rappp{} take as input three prefix queries with thresholds set at $0.49, 0.50,$ and $0.51$ and the input dataset \textsc{Real}, represented by the clear bars, that has all values equal to $0.5$. The output of both \GSD{} or \rappp{} is represented by the blue histograms. 
    \ificml
\else 
    The graph on the left shows that \GSD{} successfully produces an optimal synthetic data whose histogram exactly overlaps with the histogram of the \textsc{Real} data. On the other hand,  the right graph shows how \rappp{} fails to optimize this case by generating a significant portion of its synthetic points near $0$.
    \fi
    }
    \label{fig:badexample}
\end{figure}

%% file: Arxiv/docs/prelims.tex
\section{Preliminaries}\label{sec:prelims}

\subsection{Datasets and Queries }

Let's consider $\cX$ as a data domain of dimensionality $d$. Here, each element $x\in \cX$ is an observation within this data domain. For any given observation $x\in \cX$, we refer to its $i$-th feature as $x_i\in \cX_i$. The feature domain associated with this $i$-th feature is denoted by $\cX_i$. We define a dataset $D$ as a multiset of observations derived from the data domain $\cX$, where $D\subseteq \cX^*$. Often, to ensure a convenient representation of categorical features, we employ the one-hot encoding approach. In this context, one-hot encoding portrays a categorical value as a binary vector where all elements, barring the one corresponding to the category, are set to zero. We introduce $h_{oh}:\cX\rightarrow[0,1]^{d'}$ to denote the function that maps elements of $\cX$ to their one-hot encoding, with $d'$ representing the dimensionality of the one-hot encoded space. Finally, for any discrete set $S$, we define $U(S)$ as the uniform distribution over this set $S$. A random sample from $U(S)$ can be represented as $s\sim U(S)$. For any integer $i$, we use $[i]={1,\ldots,i}$ to represent the set of all integers from $1$ to $i$.

In this research, our focus lies on mixed-valued datasets within a bounded domain, with a particular emphasis on answering statistical queries. A statistical query is defined as a function $q:\cX^* \rightarrow [0,1]$ that maps a dataset to a value within the bounded interval [0,1]. The formal definition of a statistical query in this context is

\begin{definition}[Statistical Query \citep{kearns1998efficient}]
    A \emph{statistical query} is defined by a predicate function $\phi:\cX\rightarrow\{0,1\}$. Given a dataset $D\in \cX^N$, the corresponding statistical query $q_\phi$ is defined as : $q_\phi(D) = \frac{1}{N}\sum_{x\in D} \phi(x)$. 
\end{definition}

We use $Q=\{q_i, \ldots, q_m\}$ to represent a set of $m$ statistical queries and for a given dataset $D$.
Furthermore, for convenience,  we use
$Q(D) = [q_1(D),\ldots, q_m(D)] \in [0,1]^m$ to denote the answers of all $m$ statistical queries on $D$. 

In this work, we are interested in answering different types of statistical queries that capture different statistical properties. The first type is the class of categorical marginals, which has been the primary focus of most prior works. The \emph{$k$-way categorical marginal} queries count the fraction of rows in a dataset that matches a combination of values.


\begin{definition}[Categorical Marginal Queries]\label{def:cat_marginals}
A $k$-way \textit{categorical marginal} query is defined by a set of categorical features $S$ of cardinality $|S|=k$, together with a particular element $c\in \prod_{i\in S} \cX_i$ in the domain of $S$.  Given such a pair  $(S, c)$, let $\cX(S, c)=\{x \in \cX : x_S = c \}$ denote the set of points that match $c$. The corresponding statistical query $q_{S, c}$ is defined as $q_{S,c}(D) = \sum_{x\in D}\mathbbm{1}\{x\in \cX(S, c)\}$, 
where $\mathbbm{1}$ is the indicator function.
\end{definition}

For real-values features, we define the class of \emph{range} queries or \emph{$k$-way range-marginals} that capture the $k$.  A $1$-way range query is defined by an attribute $A$ and two real-values $a, b\in [0,1]$ and counts the number of rows where real-valued attribute $A$ lies in $[a, b]$.

\begin{definition}[Range Marginal Queries]\label{def:range_marginals}
 A $k$-way \textit{range marginal} query is defined by a set of categorical features $C$, an element $y\in \cX_C$, a set of numerical features $R$ and a set of intervals $\tau=\{[\tau_{j,0}, \tau_{j,1}]\}_{j\in R}$, with $|C| + |R| = k$ and $|R| = |\tau|$.
Let $\cX(C, y)$ be as in \cref{def:cat_marginals} and let $\cR(R, \tau) =\{x\in\cX : \tau_{j, 0} \leq x_j \leq \tau_{j, 1} \quad \forall_{j\in R}\}$ denote the set of points where each feature $j\in R$ fall below its corresponding threshold value $\tau_j$.
Then the statistical query $q_{C, y, R, \tau}$ is defined as $$    q_{C, y, R, \tau}(D) =\sum_{x\in D} 
   \mathbbm{1}\{x\in \cX(C, y) \} \cdot   \mathbbm{1}\{x\in \cR(R, \tau) \} .$$ 
\end{definition}

We note that while range marginal queries are useful for capturing statistical properties of real-valued data, they can also be applicable to discretized data. For example, one can discretize a real-valued column into the bins that are a superset of the intervals $\tau$ defined in Definition \ref{def:range_marginals}.\footnote{In Section \ref{sec:experiments} and the appendix, we will discuss how range marginal queries are still applicable to \gem{} and \rap{}, which handle real-valued attributes by discretizing them first.}

Next, we present two additional queries---halfspace and prefix queries---that are also useful in summarizing real-valued features. We emphasize that such queries are the focus of our work, since they cannot be directly optimized by existing methods using first-order optimization. These queries can in some ways be thought of as higher-dimensional generalizations of CDF functions. Moreover, synthetic data matching statistical properties defined by these queries have been shown to do well on downstream tasks like machine learning classification \citep{vietri2022private}.

\begin{definition}[Halfspace Query]\label{def:halfspace}
A \textit{halfspace} query is defined by
a vector $\theta\in\mathbb{R}^{d'}$ and a threshold $\tau\in \mathbb{R}$. The \emph{halfspace} query is given by $q_{ \theta, \tau}(D) =\sum_{x\in D}  \mathbbm{1}\{\dotp{\theta, h_{oh}(x)} \leq \tau\}$, where 
$h_{oh}:\cX\rightarrow[0,1]^{d'}$ is the function that maps elements of $\cX$ to their one-hot encoding. 
\end{definition}

\begin{definition}[$k$-way prefix queries]\label{def:prefix}
A $k$-way prefix query is defined by a set of real-valued features $C\subseteq [d]$ ($|C|=k$) and a set of threshold $\tau=\{\tau_i \in [0,1]: i\in C\}$. Let $\cX(C, \tau) =\{x\in\cX : x_i\leq \tau_i \quad \forall i\in C\}$ be the set of points, that has each feature $i\in C$ fall below the corresponding threshold $\tau_i$. Then the prefix query $q_{C, \tau}$ is given by $q_{C, \tau}(D) = \sum_{x\in D} \mathbb{I}\{x\in \cX(C, \tau)\}$.
\end{definition}


\subsection{Differential Privacy} 
To protect user privacy, we estimate $Q(D)$ under the constraint of differential privacy (DP) \cite{dwork2006calibrating}, which is a widely accepted notion of privacy protection at the user level. 
For example, the US Census Bureau has begun using differential privacy to protect the privacy of individuals when releasing census data. Differential privacy is used to add noise to the census data in order to obscure the data of any one individual, while still providing accurate statistical information about the population as a whole.
In this section, we present the formal definition of differential privacy. However, before we proceed to this definition, it is important to establish that two datasets are considered 'neighboring' if they differ by no more than a single data point.

\begin{definition}[Differential Privacy (DP) \citep{dwork2006calibrating}]\label{def:dp}
A randomized algorithm $\mathcal{M}:\cX^N\rightarrow \cR$ satisfies $(\varepsilon, \delta)$-differential privacy if 
for all neighboring datasets $D, D'$ and for all outcomes $S\subseteq \cR$ we have
\begin{align*}
    \pr{\mathcal{M}(D) \in S} \leq e^{\varepsilon} \pr{\mathcal{M}(D') \in S}+\delta
\end{align*}
\end{definition}

In this work, we also utilize a concept of {\em zero-Concentrated Differential Privacy} (zCDP) \citep{DworkR16, BunS16}, which is related to DP, but provides some advantages in terms of the analysis.  

\begin{definition}[zCDP \citep{bun2016concentrated}]\label{def:zcdp}A randomized algorithm $\mathcal{M}:\cX^N\rightarrow \cR$ satisfies $\rho$-zero-Concentrated differential privacy ($\rho$-zCDP) if 
for all neighboring datasets $D, D'$ and for all $\alpha\in(1,\infty)$ we have: 
$$\mathbb{D}_{\alpha}(\mathcal{M}(D),\mathcal{M}(D'))\leq\rho\alpha,$$ where $\mathbb{D}_{\alpha}(\mathcal{M}(D),\mathcal{M}(D'))$ is $\alpha$-Renyi divergence between the distributions $\mathcal{M}(D)$ and $\mathcal{M}(D')$.
\end{definition}

The notions of DP and zCDP are related via the following result that says that any mechanism that satisfies zCDP also satisfies DP. Therefore, in this work, we conduct the main analysis using zCDP but present the final results in terms of DP. 

\begin{theorem}[zCDP to DP]
    If $M$ provides $\rho$-zCDP, then $M$ is $(\rho+2\sqrt{\rho\log(1/\delta)},\delta)$-DP for any $\delta>0$.
\end{theorem}

An important property of DP and zCDP is composition, that allows multiple differentially private mechanisms to be combined while still maintaining a level of differential privacy. The overall privacy loss in such a composition is roughly the sum of the individual privacy losses of each mechanism. The composition property of zCDP is given by the following result:
\begin{lemma}[zCDP Composition, \citep{bun2016concentrated}]\label{lemma:comp} Let $\mathcal{A}_1:\mathcal{X}^N\rightarrow R_1$ be $\rho_1$-zCDP. Let $\mathcal{A}_2:\mathcal{X}^N\times R_1 \rightarrow R_2$ be such that $\mathcal{A}_2(\cdot,r)$ is $\rho_2$-zCDP for every $r\in R_1$. Then the algorithm $\mathcal{A}(D)$ that computes $r_1=\mathcal{A}_1(D), r_2=\mathcal{A}_2(D,r_1)$ and outputs $(r_1,r_2)$ satisfies  $(\rho_1+\rho_2)$-zCDP.
\end{lemma}

Another feature of zCDP is the post-processing property that ensures that once a mechanism meets the zCDP criterion, any subsequent processing of its output cannot compromise the privacy guarantee. This essentially implies that no additional details about the dataset can be inferred from the output.
\begin{lemma}[Post-processing]\label{lemma:post} Let $M:\mathcal{X}^N\rightarrow\mathcal{Y}$ and $f:\mathcal{Y}\rightarrow\mathcal{Z}$ be randomized algorithms. Suppose $M$ satisfies $\rho$-zCDP. Define $M':\mathcal{X}^N\rightarrow\mathcal{Z}$ by $M'(X)=f(M(X))$. Then $M'$ satisfies $\rho$-zCDP.
\end{lemma}

%
%

{Next, we describe the basic private tools used in this work. The first one is the Gaussian mechanism for privately estimating a selected set of statistical queries. This mechanism operates by introducing Gaussian noise to the computed statistics. The scale of the Gaussian noise is determined by the sensitivity of the queries—defined as the maximum variation in a query's value attributable to a single alteration in the dataset. More formally, consider a vector-valued query $Q: \mathcal{X}^N \rightarrow [0,1]^m$. The $\ell_2$-sensitivity of this query, denoted as $\Delta(Q)$, it's calculated as $ \Delta(Q) = \max_{\text{neighboring } D,D'} \| Q(D) - Q(D')\|_2$, where neighboring datasets $D$ and $D'$ differ by a single element. 
}

\begin{definition}[Gaussian Mechanism]\label{def:gaussian}
The Gaussian mechanism takes as input a dataset 
$D\in \cX^N$,  a set of $m$ statistical queries $Q=\{q_1,\ldots, q_m\}$, and a privacy parameter $\rho>0$. And outputs $Q(D) + \cN\pp{0, \tfrac{\Delta(Q)^2}{2\rho}\cdot \mathbb{I}_m}$
%
The Gaussian mechanism satisfies $\rho$-zCDP \citep{bun2016concentrated}.
\end{definition}

We will also use the private selection algorithm \emph{report noisy max} \citep{durfee2019practical} with Gumbel perturbation noise, whose output distribution is identical to that of the exponential mechanism \citep{mcsherry2007mechanism}. Next, we state both definition is the context of answering statistical queries.


\begin{definition}[Exponential Mechanism]\label{def:exponential}
The Exponential mechanism $\mathcal{M}_E(D, Q,\varepsilon)$ takes as input a dataset $D\in \cX^N$, a set  a set of $m$ statistical queries $Q=\{q_1,\ldots, q_m\}$,  a vector of $m$ conjectured query answers $\hat{a}\in [0,1]^m$ for each query in $Q$, and a privacy parameter $\varepsilon$.
Then, for each query $q_i\in Q$, it defines a scoring function $S(q_i) =  |q_i(D) -  \hat{a}_i|$, where the function $S$ has sensitivity $\Delta(S) = \frac{1}{N}$.
Finally, it outputs an element $q_i\in Q$, with probability proportional to $\exp\pp{\frac{\varepsilon\cdot S(D, q_i)}{2\Delta(S)}},$
 The Exponential mechanism $\mathcal{M}_E(D,S,\mathcal{R},\varepsilon)$, satisfies $(\frac{\varepsilon^2}{8})$-zCDP \citep{cesar2021bounding}
 \end{definition}

\begin{definition}[Report Noisy Max With Gumbel Noise]\label{def:rnm} The Report Noisy Max mechanism $RNM(D,q,a,\rho)$ takes as input a dataset $D\in\mathcal{X}^N$, a set of $m$ statistical queries $Q=\{q_1,\ldots, q_m\}$, a vector of $m$ query answers $\hat{a}$, and a zCDP parameter $\rho$. It outputs the index of the query with highest noisy error estimate, $i^*=\arg\max_{i\in[m]}(|q_i(D)-a_i|+Z_i)$ where each $Z_i\sim{\rm Gumbel}(1/\sqrt{2\rho}N)$. 
The Report Noisy Max With Gumbel Noise $RNM(\cdot,q,a,\rho)$ satisfies $\rho$-zCDP.
\end{definition}

%% file: Arxiv/docs/query_release_problem.tex
\section{Private Query Release via Synthetic Data} \label{sec:projmech}

For the problem of estimating statistics $Q(D)$ subject to differential privacy, we follow the framework of the projection mechanism \citep{NKL, DworkNT15}, which consists of perturbing the statistics with a differentially private mechanism followed by a \emph{projection step} that consists of finding a synthetic dataset that matches these perturbed statistics.

In this study, we consider two versions of the private projection mechanism. The first is the standard projection mechanism, referred to as the \emph{one-shot} version. The second is an \emph{adaptive} version of the projection mechanism, which proves useful in scenarios where a large number of queries are required. A comprehensive description of both frameworks is presented below.

\paragraph{One-shot projection mechanism:}
Suppose we are given $m$ statistical queries, represented by $Q$, and a dataset $D$. Then the first step consists of independently perturbing every coordinate of $Q(D)$ to obtain a private estimate, denoted by $\hat{a}$. Given a privacy parameter $\rho$, the noisy answer vector is computed as follows:
\begin{align}\label{eq:gaussianmech}
    \hat{a} \leftarrow Q(D) + \cN\pp{0, \tfrac{\Delta(Q)}{2\rho}\cdot \mathbb{I}}
\end{align}

Then, the \emph{project} step involves finding a synthetic dataset $\Dhat$ that best explains the estimate $\hat{a}$. To represent the projection step, we construct a loss minimization problem, where the loss function is defined by $\hat{a}$ and $Q$ as follows:
\begin{align}\label{eq:obj}
 L_{\hat{a}, Q} (\Dhat) =   \|\hat{a} - Q(\Dhat) \|_2^2
\end{align}

By the properties of the Gaussian mechanism and the post-processing property (Theorem \ref{lemma:post}), the projection mechanism framework satisfies  $\rho$-zCDP regardless of the optimization procedure used to solve \eqref{eq:obj}.

\paragraph{Adaptive selection framework:} In practice, when the number of queries $m$ is very large, the objective \eqref{eq:obj} becomes extremely noisy for any practical optimization routine. Therefore, we follow the Adaptive Selection Framework of \citet{liu2021iterative} that, over $K$ rounds, chooses high error queries adaptively using  the report noisy max mechanism (alternatively the exponential mechanism, as defined in Definition \ref{def:exponential}). Algorithm \ref{alg:adaptive} of Appendix \ref{appx:adaptive} shows the details of this framework, which we note satisfies the following theorem:

\begin{theorem}\label{thm:privacy} For any $K$, $\rho>0$, dataset $D$, query set $Q$ and any procedure used to solve objective \eqref{eq:obj},  the adaptive framework  strategy satisfies $\rho$-zCDP. By extension, it also satisfies $(\varepsilon, \delta)$-differential privacy for: $\varepsilon = \rho+2\sqrt{\rho\cdot \log(1/\delta)}$.
\end{theorem}
We provide a proof sketch to Theorem \ref{thm:privacy} below.
\begin{proof}
In the adaptive selection framework, the report noisy max (RNM, Definition \ref{def:rnm}) and Gaussian mechanism (GM, Definition \ref{def:gaussian}) are both called $K\cdot S$ times, each with a privacy budget of $\rho' = \rho / (2\cdot K\cdot S)$. As shown in Definition \ref{def:gaussian} and the work of \citet{durfee2019practical}, each individual call to RNM or GM satisfies $({\rho'})$-zCDP. Thus, by the composition properties outlined in \citet{bun2016concentrated}(see Lemma \ref{lemma:comp} ), the combination of $2\cdot K\cdot S$ $\rho'$-zCDP mechanisms satisfies $\rho$-zCDP.\end{proof}

%% file: Arxiv/algorithms/gsd.tex
\begin{algorithm}[!htb]
\caption{Private Genetic Algorithm for Synthetic Data ($\GSD$) \label{alg:GSD}}
\begin{algorithmic}[1]
\STATE \textbf{Require:} The search space $\cX^{N'}$ of mixed-type synthetic datasets with $N'$ rows and $d$ features, number of generations $G$, mutations population size $\Pmut$, crossover population size $\Pcross$ and elite  set size $E$.
\STATE \textbf{Input:} A set of $m$ queries $Q:\cX^* \rightarrow [0,1]^m$, (noisy) query answers $\hat{a}\in [0, 1]^m$.

\STATE Set the objective function for the synthetic datasets:
\begin{align*}
    L_{Q, \hat{a}}(\Dhat) = \| \hat{a} - Q(\Dhat)\|_2
\end{align*}

\STATE Initialize uniformly at random the elite set $\elite_1 \leftarrow \{\Dbar_{1, i} \sim U(\cX^{N'}): i \in [E]\}$ of synthetic datasets with $N'$ rows.
\FOR{$g = 1, \ldots, G$}
    
    \STATE Set $\Dhat^\star_{g} \leftarrow \argmin_{\Dhat \in \elite_{g}} L_{Q, \hat{a}}(\Dhat)$.

    \STATE Initialize candidate population set $\cP_g \leftarrow\elite_g$. 
    \STATE \textbf{Mutations: } Repeat $\Pmut$ times:  Initialize a synthetic data $\Dtil$ equal to $\Dhat_g^\star$. Then, sample a row $i\sim U([N'])$, feature $j\sim U([d])$ and value  $v \sim U\pp{\cX_j}$.  Update  $\Dtil$ as follows:
        $ \Dtil(i, j) \leftarrow v$ and add $ \Dtil$ to $\cP_g$.

    \STATE \textbf{Crossovers:} Repeat $\Pcross$ times: Initialize $\Dtil$ equal to $\Dhat_g^\star$. Then, sample two rows  $i_1, i_2\sim U([N'])$,  a feature $j\sim U([d])$, and elite member $\Dbar \sim U(\elite_g)$ and update  $\Dtil$ as follows:  $\Dtil(i_1, j)  \leftarrow \Dbar(i_2, j) $. Then add $\Dtil$ to $\cP_g$. 
        
    
    
    \STATE \textbf{Update Elites:} Evaluate each member of $\cP_g$ on $L_{Q, \hat{a}}(\cdot)$. Then, set $\elite_{g+1}$ as top $E$ members from population $\cP_g$  with respect to \emph{minimizing} objective $L_{Q, \hat{a}}(\cdot)$.
\ENDFOR
\STATE \textbf{return } $\Dhat_G^\star$
\end{algorithmic}
\end{algorithm}

%% file: Arxiv/docs/PrivGA_2.tex
\section{Private Genetic Synthetic Data (\privGA)}\label{sec:GSDdesc}

The generation of high-quality synthetic data through the projection mechanism framework, as described in \cref{sec:projmech}, presents a challenging task that necessitates the solution of an NP-hard optimization problem \citep{HT10}.  Therefore, we propose the \privGA{} algorithm, which solves the projection step in the projection mechanism using a genetic algorithm (GA). 

\subsection{Background on genetic algorithms}
\label{sec:gabackground}

Genetic Algorithms (GA) are a class of optimization algorithms inspired by natural selection and genetic recombination. Introduced by \citet{holland1992genetic}, GAs provide a heuristic search technique to solve complex optimization problems by evolving a population of candidate solutions towards an optimal or near-optimal solution. A crucial advantage of GAs in this work is that they do not require differentiability of the optimization objective.

GAs offer several benefits for optimization problems; however, GAs do not represent a one-fit-all solution since they require a suitable problem representation (i.e., encoding) and choice of genetic operators to work effectively. For instance, the choice of problem representation can significantly affect the GA's computational efficiency because they typically require multiple iterations and evaluations of the objective function, which can be time-consuming. Furthermore, failure to choose suitable genetic operators could lead to poor candidates or a lack of diversity in the population, making the GA suffer from slow convergence or convergence to the local minimum.

\subsection{\GSD{}}

This section presents a genetic algorithm named  Private Genetic Synthetic Data (\GSD), developed specifically for synthetic data generation.
\GSD{} applies two genetic operators known as `mutations' and `crossover' in the GA literature (\citep{such2017deep}). 
In contrast to prior genetic algorithms, the mutation and crossover operator applied by \GSD{} are specifically designed for searching the space of synthetic datasets and optimizing the projection step  described in \cref{sec:projmech}. 
These operators, aimed at generating incremental improvements in synthetic datasets across multiple generations, have been designed for computational efficiency and quick convergence.  We begin by unpacking the core elements and structure of \GSD{} and then dive into the specifics of these genetic operators.

The \GSD{} algorithm operates on a set of statistical queries, denoted as $Q$, and their respective target private responses, $\hat{a}$. It formulates an optimization problem via the objective function $L_{Q,\hat{a}}:\cX^{N'}\rightarrow [0,1]$. A proposed solution to $L_{Q,\hat{a}}$ is a synthetic dataset comprising a user-specified number of rows, $N'$. Thus, \GSD's ultimate objective is to identify a synthetic dataset within $\mathcal{X}^{N'}$ that minimizes $L_{Q,\hat{a}}$. \GSD{} also incorporates other parameters such as the maximum number of generations, $G$, over which it will operate, as well as $\Pmut$ and $\Pcross$, which respectively represent the quantity of proposed solutions generated using mutation and crossover strategies. Lastly, the $E$ parameter determines the number of candidates that will be selected to parent the succeeding generation.

The genetic process unfolds across $G$ generations. For each generation, denoted as $g\in [G]$, \GSD{} maintains a set of elite $E$ candidate solutions, symbolized as $\elite_g$. These represent the optimal $E$ candidate solutions discovered thus far. The initial elite set, $\elite_1$, is chosen at random from the set $\cX^{N'}$. Within each generation, \GSD{} generates $(\Pmut+\Pcross)$ new candidate solutions. These are obtained by introducing minor random variations to the best synthetic data identified up to generation $g$, denoted as $\Dhat^\star_g$. To achieve this, \GSD{} applies a random genetic modification—either mutation or crossover—to $\Dhat^\star_g$, yielding a new synthetic data candidate, $\Dtil$. This process allows for inheriting traits from the elite synthetic dataset and adds diversity, paving the way for potentially superior solutions.

The group of candidate synthetic datasets in generation $g\in[G]$, denoted as $\cP_g$, comprises the $\Pmut+\Pcross$ new candidates and the elite candidates $\elite_g$. \GSD{} assigns a score to each candidate in $\cP_g$ based on their alignment with the objective function $L_{Q, \hat{a}}$. It then ranks all $\Pmut+\Pcross+E$ individual candidates according to their scores, and the top $E$ candidates with the highest fitness scores are chosen as the next generation of elite candidates. This iterative procedure continues until a predetermined termination criterion—such as reaching the maximum number of generations, achieving a satisfactory fitness level, or meeting a predefined condition—is satisfied.

Now, we delve into the specifics of the `mutation' and `crossover' genetic operators, which play a significant role in \GSD's performance. These operators are only applied to the optimal dataset $\Dhat_g^\star$, with the goal of incremental improvements over many generations. Moreover, both operators only modify a single value of $\Dhat_g^\star$. Empirically, we find that having sparse genetic updates is crucial to the success of \GSD, as shown in \cref{tab:parameters} where altering numerous values simultaneously performs noticeably worse.

\paragraph{Mutation.}
The mutation process is a vital part of a genetic algorithm. It introduces minor random perturbations to the most optimal solution to create novel candidate solutions. Controlled by parameter $\Pmut$, the mutation strategy of the \GSD{} algorithm determines the number of candidate synthetic datasets produced in each generation.
During the mutation phase, a row $i \in [N']$ and column $j \in [d]$ are randomly selected. The entry in $\Dhat^\star_g$ corresponding to the selected row $i$ and column $j$ is then replaced by a value sampled uniformly from the range of acceptable values for column $j$. This modified dataset, which differs from $\Dhat^\star_g$ by a single entry,\footnote{While we could adjust the number of mutated rows at each step, our empirical data shows that in our settings, mutating a single row often yields the best performance (see Table 1).} is subsequently added to the pool of candidate solutions.

\paragraph{Crossover.}
Taking inspiration from genetic recombination, the crossover operator merges the parameters of two parent candidates to generate a new offspring. Parameter $\Pcross$ determines the number of candidate synthetic datasets produced in each \GSD{} generation via the crossover strategy.
To sample a new candidate synthetic dataset, the \GSD's crossover operation selects an elite dataset $\Dbar$ from the elite set $\elite_g$. It then combines the parameters of $\Dbar$ with those of the optimal dataset $\Dhat^\star_g$ to produce a new dataset $\Dtil$. This synthetic dataset, differing from $\Dhat^\star_g$ by a single entry, is subsequently added to the candidate solutions population, $\cP_g$.
The crossover operation combines the parameters of $\Dhat^\star_g$ and $\Dbar$ by selecting two rows $i_1, i_2\in[N']$, and a feature $j\in[d]$. The new dataset $\Dtil$ then has its $(i_1, j)$ entry set to match the value of $\Dbar$ at row $i_2$ and feature $j$ (i.e., $\Dtil(i_1, j) = \Dbar(i_2, j)$), while all other entries of $\Dtil$ mirror the corresponding values at $\Dhat^\star_g$.

The crossover operator plays an integral role in improving the convergence rate of \GSD. \cref{fig:convergence} shows the convergence of the \GSD{} algorithm with and without the crossover operator. By encouraging the replication of significant values (with respect to the fitness function) in subsequent generations, \GSD{} with the crossover step converges exponentially faster compared to an algorithm reliant solely on mutations.

\input{Figures/convergence}

\subsection{Privacy and Accuracy}

Below, we state the formal privacy guarantee of \privGA{}, which directly follows from \cref{thm:privacy}.

\begin{corollary}\label{thm:privgaguarantee} 
For any $\rho>0$, dataset $D$, and query set $Q$.
Running \privGA{} (\cref{alg:GSD}) under the adaptive framework satisfies $\rho$-zCDP and  $(\varepsilon, \delta)$-differential privacy for: $\varepsilon = \rho+2\sqrt{\rho\cdot \log(1/\delta)}$.
\end{corollary}

Given that \GSD{} falls within the `projection mechanism' framework proposed by \citet{NKL}, an accuracy analysis can be performed under several assumptions: (1) the input data are discrete, (2) the number of queries is finite, and (3) the optimization step invariably finds the optimal solution of the objective. This kind of analysis assumes access to a perfect optimization oracle and is often referred to as `oracle-efficient' accuracy analysis. It follows a similar approach to previous studies by \citet{gaboardi2014dual, vietri2022private, aydore2021differentially}.

The accuracy statement for \GSD{} is presented in \cref{thm:accuray}, with a detailed proof provided in the \cref{thm:accuray}. Our analysis of \GSD's accuracy is similar to the accuracy analysis of the \RAP{} mechanism \citep{aydore2021differentially}, which also uses the projection mechanism framework and adopts the `oracle-efficient' assumption. However, unlike \RAP, which implements a relaxation step in the data domain to maintain objective differentiability, \GSD{} does not perform such a step. Consequently, our accuracy theorem applies to a general class of queries, whereas the one for \RAP{} is limited to $k$-way marginal queries.

     

\begin{theorem}\label{thm:accuray}
For a discrete data domain $\cX$ and any dataset $D\in\cX^N$ with $N$ rows.
Fix privacy parameters  $\epsilon, \delta>0$, the synthetic dataset size $N'$, and any set of $m$ queries $Q:\cX^*\rightarrow[0,1]^m$. If the one-shot \privGA{} mechanism solves the minimization in the projection step exactly, then 
the one-shot \privGA{} mechanism outputs a synthetic data $\Dhat$ with
 average error bounded as
\begin{align*}
  \sqrt{\tfrac{1}{m} \| Q(D) - Q(\Dhat) \|_2^2} \leq 
   O\pp{
   \frac{
    \pp{(\pp{\log(|\cX|) + \log(1/\beta)}\ln(1/\delta)}^{1/4}
   }{\sqrt{\epsilon\cdot N }}
   + \frac{\sqrt{\log(m)}   }{\sqrt{N'}}
   }
\end{align*}
\end{theorem}

\begin{remark}
\cref{thm:accuray} only applies if the underlying optimization algorithm in \cref{alg:GSD} is successful in minimizing the objective in \cref{eq:obj}. Therefore, we accompany the accuracy theorem with empirical evidence that validates the performance of \privGA{} across various datasets.    In \cref{fig:oneshot_2way_cat}, we run the one-shot version of \privGA{}  on various discrete datasets to match the set of all $2$-way marginal statistics. For more details see 
\cref{sec:experiments}.
\end{remark}

%% file: Figures/convergence.tex
\begin{figure}
    \centering
    \ificml
\else
    \includegraphics[width=0.40\linewidth]{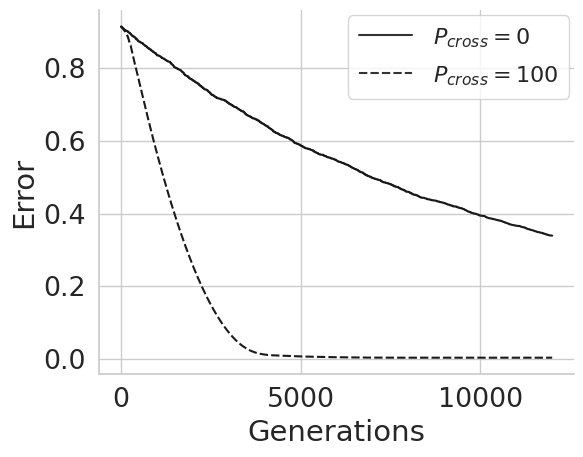}
    \fi
     \caption{
     A comparison of the convergence rates of \GSD{} under two different parameterizations, aiming to understand the impact of the \emph{crossover} genetic operator. 
     The accompanying plot illustrates the error of \GSD{} after each generation, in relation to matching the histogram distribution (1-way marginals) of the input data.  The the solid line represents parameters $\Pmut=200$ and $\Pcross=0$, while the dotted line corresponds to $\Pmut=100$ and $\Pcross=100$.
     The experiment employs an artificially created dataset, intentionally designed to be sparse and of high cardinality, thus challenging to approximate. The dataset comprises five categorical features, each with a cardinality of $1000$, and contains $2000$ rows.
     }
\label{fig:convergence}
\end{figure}

%% file: Arxiv/docs/experiments.tex
\section{Experiments}\label{sec:experiments}

We conduct experiments to compare the performance of \privGA{} against baseline mechanisms in various settings. In particular, we evaluate the performance of synthetic data mechanisms on mixed-typed data, which contain both categorical and numerical attributes. Our results demonstrate that in this setting, \privGA{} is a strong mechanism in approximating various types of statistical queries, outperforming \rappp{}, the state-of-the-art method for generating synthetic data with real-valued columns. In addition, we demonstrate that \privGA{} achieves performance comparable to that of mechanisms specifically designed for---and restricted to---handling differentiable query functions for discrete data.





\subsection{Data}\label{sec:data}

For our empirical evaluation, we use datasets derived from the Folktables package \cite{ding2021retiring}, which defines datasets using samples from the American Community Survey (ACS). The ACS is collected annually by the U.S. Census Bureau to capture the demographic and economic characteristics of individuals and households in the country. In particular, Folktables provides an array of classification tasks, which are defined by a target column, a collection of mixed-type features, and some filter that specifies what samples to include.\footnote{For example, a task can be defined to include only rows with individuals over the age of 18.}
%
%
The experiments in this section use samples from California and the five ACS tasks listed in \cref{tab:data_stats}. Finally, each dataset is normalized by linearly scaling real-valued columns to lie in $[0,1]$.

\begin{table}[ht]
\centering
\label{tab:data_stats}
\begin{tabular}{l | c c c }
    \toprule
    Dataset & N & \# Real & \# Cat. \\
    \midrule
    \acsmobdata{CA} & 64263 & 4 & 17 \\
    \acsempdata{CA} & 303054 & 1 & 16 \\
    \acsincdata{CA} & 156532 & 2 & 9 \\
    \acscovdata{CA} & 110843 & 2 & 17 \\
    \acstradata{CA} & 138006 & 2 & 14 \\
    \bottomrule
\end{tabular}
\caption{For each dataset, we list the number of rows $N$, as well as the number of real and categorical features.}
\end{table}



\subsection{Statistical queries} 

Statistical queries are essential for data analysis and decision-making, as they allow us to extract meaningful insights from large and complex datasets. In this section, we define the four sets of statistical queries that we use in our experiments. We summarize these queries in \cref{tab:queries}  and describe them in more detail below:

\paragraph{Random Prefixes:} 
The first query class we consider is a random set of $50,000$ prefix queries, as in \cref{def:prefix}.
To generate a random prefix query, we randomly sample a categorical column $c$ and two numeric columns $a, b$. Then, uniformly sample a target value $v$ for column $c$ and the two thresholds $\tau_a,\tau_b\in[0,1]$ for the numeric columns $a$ and $b$, respectively. The corresponding prefix query is defined as follows:
\begin{align*}
    q_{(c, v),  (a, \tau_a), (b, \tau_b)}(D) = \sum_{x\in D}\mathbb{I}\{x_c = v \text{ and } x_{a}< \tau_a \text{ and } x_b < \tau_b\}
\end{align*}
%

\paragraph{Random Halfspaces}
Next, we sample $m=200,000$ random halfspace queries as in \cref{def:halfspace}. Each sample is generated as follows:   Let $d'$ be the dimension of the one-hot encoding of $\cX$, determined by the one-hot encoding function $h_{oh}(\cdot)$. First, sample a random vector ${\theta}\in \mathbb{R}^{d'}$, where the value of each coordinate $i\in [d']$ is an i.i.d sample drawn from the normal distribution with variance $\tfrac{1}{d}$ (i.e., ${\theta}_i \sim \cN(0, \tfrac{1}{d})$). Then, sample a threshold value $\tau$ from the standard normal distribution, i.e., $\tau\sim \cN(0,1)$. Finally, the corresponding query   $q_{\theta, \tau}$ is defined as follows:
\begin{align*}
    q_{\theta, \tau}(D) =\sum_{x\in D} \mathbb{I}\{\langle \theta, h_{oh}(x)\rangle \leq  \tau\}
\end{align*}



\paragraph{$k$-way Categorical Marginals:}
Using categorical columns only, we select all $2$-way marginals in the one-shot setting and all $3$-way marginal queries in the  adaptive setting. Since the query answers in a workload sum up to $1$, each workload has $\ell_2$-sensitivity $\Delta_2 = \sqrt{2}$.

\paragraph{$k$-way Binary-Tree Marginals:}
In this work we introduce a class of queries, which we coined \emph{Binary-Tree Marginals}. The Binary-Tree marginals consists of a collection of range queries, as in \cref{def:range_marginals}) defined over various levels. 


We construct range queries using one categorical attribute and $k-1$ real-valued attribute. Recall that real-valued features lie in the domain $[0,1]$. Thus, we define a workload of range queries by the set of intervals of width $\frac{1}{2^j}$ for $j \in \{1, 2, \ldots, 5\}$.
Now consider the set
\begin{align*}
    I = \{ \pp{\tfrac{i}{2^j}, \tfrac{i + 1}{2^j}}| 0\leq i< 2^j , j \in \{1, 2, \ldots, 5\} \}
\end{align*}
Using the same notation from definition \ref{def:range_marginals}, for real-valued feature $c$, the corresponding range workload  $Q_{c, I} =\{ q_{c, (a,b)}\}_{(a, b) \in I}$ is the set of range queries defined by intervals in $I$. This construction simulates the binary mechanism of \citet{chan2011private} for releasing sum prefixes privately.

Note that a single data point only appears in one interval of size $1/2^j$ for each $0< j \leq 5$. Therefore, adding a new data point only affects the counts of $5$ range queries in the workload $Q_{c, I}$.
By the same logic, replacing a data point affects $10$ range queries in  $Q_{c, I}$. Therefore, we say that the workload  $Q_{c,I}$ has $\ell_2$-sensitivity $\sqrt{10}$.


\begin{table}[!htb]
\centering
\begin{tabular}{l |c c}
    \toprule
    Query class & Query set size & $\Delta_2$ \\
    \midrule
    2-way Categorical & ${\binom{d}{2}}$ & $\sqrt{2}$ \\  
    \midrule 
    3-way Categorical & ${\binom{d}{3}}$ & $\sqrt{2}$ \\  
    \midrule 
    2-way Binary-Tree & $d_{cat} \cdot d_{num}$ & $\sqrt{10}$ \\
    \midrule
    Random Prefix & $200000$  & $1$ \\  
    \midrule
    Random Halfspaces & $200000$  & $1$ \\  
    \bottomrule
\end{tabular}
\caption{The total number of queries depends on the input data total number of features, which we denote by $d$ and  the number of categorical and numerical features, $d_{cat}$ and $d_{num}$ respectively. In addition, we list the $\ell_2$-sensitivity $\Delta_2$ of each query class (the sensitivity for categorical and range marginals is written with respect to an entire workload of queries). Note that in our experiments, the number of categorical attributes comprising both $2$- and $3$-way range marginals is always set to $1$.}
\label{tab:queries}
\end{table}


\subsection{Baselines}

\paragraph{\rappp.}

We compare our mechanism \privGA{} against the mechanism \rappp{} by \citet{vietri2022private}. The \rappp{} algorithm uses a first-order optimization algorithm called {\em Sigmoid Temperature Annealing} (pseudocode can be found in Appendix \ref{appx:rappp}). Recall that \rappp{} replaces the prefix queries with a differentiable surrogate query by using Sigmoid functions. Moreover, to control the degree by which the Sigmoid functions approximate the prefix functions, it uses the {\em inverse Sigmoid temperature parameter} $\invtemp$. The algorithm starts with a small inverse temperature $\invtemp$ and runs a gradient-based minimizer (SGD or Adam) on the induced optimization objective with learning rate \LR, and then repeats the process for a different choice of inverse temperature. 




\paragraph{\rp{} / \gn{}.}

We also compare our method to \rap\footnote{We use the softmax variant proposed by \citet{liu2021iterative}, which has been shown to perform much better that the versions originally proposed in \citet{aydore2021differentially}. We note in \citet{vietri2022private}, references to \rap{} also refer to this softmax version.} \citep{liu2021iterative, aydore2021differentially} and \gem{} \citep{liu2021iterative}, two gradient-based optimization algorithms for generating private synthetic data that belong to discrete data domains. In both algorithms, datasets are modeled as mixtures of product distributions $P$ over the attributes, with the main difference being the way in which such distributions are parameterized---in \gem{}, $P$ is parameterized by a neural network that given Gaussian noise, outputs a set of product distributions, while in \rap{}, $P$ is parameterized directly by the probabilities defining the distribution for each attribute. Because both methods are limited to discrete data, real-valued columns must be discretized first. Consequently, \rap{} and \gem{} can only optimize over $k$-way categorical and binary-tree marginals. In \cref{appx:rapgem}, we show how we modify these methods to optimize over binary-tree queries and also provide additional details about their optimization procedures. Going forward, we denote the one-shot version of \rap{} and \gem{} by \rp{} (RelaxedProjection) and \gn{} (GenerativeNetwork) respectively.

\paragraph{\pgmem.}

Lastly, we compare to a graphical model approach, \pgm{} \citep{mckenna2019graphical}, which like \rap{} and \gem{}, is also limited to discrete data domains. However, \pgm{} instead models data distributions using probabilistic graphical models. We run an adaptive version of \pgm{} that also uses the exponential mechanism. We refer to this variant as \pgmem{} and run it on $k$-way categorical marginals.

\input{Figures/prefix}
\input{Figures/halfspace}

\input{Figures/oneshot_categorical}

\input{Figures/oneshot_ranges}

\input{Figures/adaptive_3way_categorical}

\subsection{Early Stop} 

In all experiments, \GSD{} operates over a maximum of $G$ generations. However, we employ an early-stop rule if the algorithm has reached a satisfactory level of convergence in optimizing its objective. This not only aids in achieving the desired result more swiftly but also conserves computational resources.

To determine the convergence of the algorithm, we utilize a time-window approach. We define a window size, $w$ (where $w < G$), for assessing the early-stop condition. The algorithm is halted if the percent change in loss between the current generation, $g$, and the generation $g-w$ is less than a predetermined threshold. In all cases in this paper, the threshold is set to $0.0001$. If the change in loss between the two generations within the window falls below this threshold, the algorithm is halted prematurely. This allows us to conclude that the algorithm has sufficiently converged and further iterations would not significantly contribute to the optimization of the objective. Additionally, the size of the time-window, $w$, is set to match the size of the synthetic dataset, $N'$. The proportionality of the window size to the dataset size ensures that our early-stop condition is adapted to the scale of the problem at hand.

\subsection{Results}

We now present empirical results comparing \privGA{} to various baseline methods. For our experiments, we fix $\delta=\frac{1}{N^2}$ and $\varepsilon\in\{0.07, 0.23, 0.52, 0.74, 1.0\}$ to examine the trade-off between privacy and accuracy. For adaptive algorithms, we run using adaptive epochs $T \in \{25, 50, 75, 100\}$ and report the best error (averaged over 3 runs) for each $\varepsilon$ value.

First, we provide empirical evidence that \privGA{} is a better algorithm for generating real-valued synthetic data (i.e., when the statistical queries are non-differentiable). As originally shown in \citet{vietri2022private}, Figures \ref{fig:errors_prefix} and \ref{fig:errors_halfspace} demonstrate that \rappp{}, the previous state-of-the-art method, can optimize over prefix and halfspace queries using a surrogate differentiable approximation in combination with temperature annealing. However, by directly optimizing over such objectives using zeroth order optimization, \privGA{} significantly outperforms \rappp{} with respect to max and average error. 

Next, we show that on discretized data (both categorical and mixed-type), \privGA{} mechanism performs equally as well as prior work \rp{}, \gn{} and \pgmem. In Figures \ref{fig:oneshot_2way_cat} and \ref{fig:ada_3way_cat}, we evaluate on categorical marginals, which are one of the primary focuses of \rp{}, \gn{} and \pgm{}.\footnote{We note that \pgm{} cannot be run with the one-shot mechanism since the size of graphical model becomes to large.} In addition, we evaluate on binary-tree queries in Figure \ref{fig:oneshot_2way_bt}. We note that for \rp{} and \gn{}, we discretize the data in such a way that the resulting bins align with the intervals comprising our binary-tree marginal queries, allowing us to also write differentiable forms of such queries for \rp{} and \gn{}.
Nevertheless, we observe that \privGA{} is able to synthesize both categorical-only (for categorical marginals) and mixed-type (for binary-tree marginals) data that matches the performance of the three baseline algorithms. Moreover, unlike the baseline methods, \privGA{} does not require that the numerical features be discretized first when optimizing over binary-tree marginal queries.

\subsection{ML Evaluation}
This section critically assesses the utility of synthetic data in training machine learning models. We explore two main parameters in our investigation: the choice of mechanisms, such as \GSD{} or \rappp, and the selection of statistics to match, including categorical marginals, Binary Tree, or Halfspace queries.

We direct our attention toward datasets that encapsulate multiple classification tasks concurrently. To this end, we construct a multitask dataset that amalgamates all five predictive tasks outlined in \cref{fig:MLf1} by combining all feature and target columns. The learning challenge presented by a multitask dataset involves developing separate models, each predicting one of the target columns based on the feature columns. We exclude the employment task because its target is highly correlated with feature from another tasks, making the prediction problem trivial. Thus, the multitask dataset incorporates 25 categorical features, nine numerical features, and four binary target labels.

We initiate the experiment by dividing each dataset into a training and test set, using an 80/20 partition. The training set is subsequently processed by a private mechanism to produce a synthetic dataset that we use to train a logistic regression model. Lastly, the performance of the model is evaluated on the test data using the F1 metric. The F1 metric considers both precision and recall, providing a balanced assessment of model performance, particularly when dealing with imbalanced datasets. It is calculated as the harmonic mean of precision and recall, with a potential range from $0$ to $1$, wherein a higher score denotes superior model performance.

Our results are tabulated in Table ML and showcase the F1 score for predicting four target labels using logistic regression trained on various synthetic data methodologies alongside the F1 score of a logistic regression model trained on the original training data. The findings suggest that the \GSD{} mechanism coupled with Binary-Tree queries is a favorable choice among the other synthetic data baselines examined in this experiment. However, we note that the F1 scores of all private synthetic data methodologies are inferior to the model trained on the original data, even in a low privacy regime (i.e., $\epsilon=1$).

The question of why synthetic data does not get better performance in training ML models remains an open area of exploration. One conjecture is that we have yet to identify the right set of data properties (or statistical queries) that can be used to generate high-quality synthetic data for machine learning applications. However, we posit that the \GSD{} mechanism is well-equipped to tackle this challenge due to its inherent ability and flexibility in managing diverse queries.

\begin{figure}[!htb]
    \centering
    \includegraphics[width=0.7\linewidth]{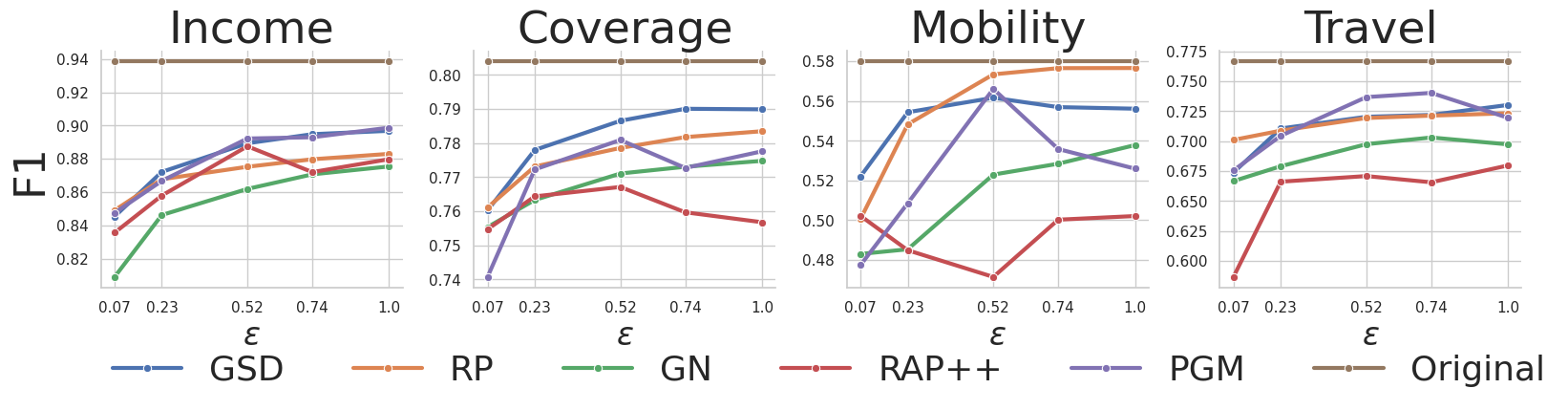}
    \caption{
    \textbf{Machine Learning Evaluation: } F1 test score of logistic regression trained on synthetic data. Each ACS dataset is partitioned into training and testing set, using an 80/20 split. The train set is used to train the synthetic data algorithm, and the F1 score is computed using the test set.
    Each algorithm is trained either on Binary Tree (BT) marginal queries or Halfspace (HS) queries.   We fix $\delta={1/N^2}$ and vary $\varepsilon\in\{0.07, 0.23, 0.52, 0.74, 1.0\}$. Results are averaged over 3 runs. }
    \label{fig:MLf1}
\end{figure}

\subsection{Ablation study on \GSD{} parameters}\label{sec:ablationstudy}
This sections analyzes different configurations and parameters of \GSD{} and their effect on the final approximation error and runtime. As described in \cref{alg:GSD}, the relevant parameters are the number of generations $G$, the number of mutations $\Pmut$, the number of crossovers $\Pcross$ and the elite size $E$. 

\paragraph{Population size} In this section, we delve into the investigation of parameters $\Pmut$ and $\Pcross$, which guide the generation of candidate synthetic datasets in each iteration. $\Pmut$ and $\Pcross$ represent the count of candidates created via the mutation and crossover strategies respectively. Consequently, these population size parameters correspond to the exploration rate of the \GSD{} method; larger population sizes correspond to a broader array of potential solution candidates at each generation.

Nonetheless, a trade-off exists: larger population sizes incur higher computational costs since \GSD{} to computes the objective function for synthetic dataset candidate. If we denote the quantity of queries as $m$ and the population size as $P$, the computational cost of \GSD{} for each iteration equals $O(P\cdot m)$. Conversely, a reduced population size $P=\Pmut+\Pcross$ may result in slower convergence.

To identify the optimal population size, we conduct an experiment where we adjust the population size and observed its impact on the final error and computation time. We run \GSD{} that outputs a synthetic dataset matching all $2$-way marginal queries on five separate ACS datasets. For each dataset, we initiate \GSD{} with different $\Pmut$ and $\Pcross$ values, maintaining $\Pmut \leq \Pcross$ and keeping other parameters constant. The privacy parameter $\epsilon$ is set to $1$, synthetic data rows to $N'=1000$, and the number of iterations to $G=300,000$ with early stopping.

We present our experimental findings in figures \cref{tab:ablationsummary} and \cref{fig:ablation_runtime}. \cref{tab:ablationsummary} offers insights into the impact of the population size parameters on the final error and runtime. Notably, while the population size parameters only moderately affect the final error, they substantially influence the computational time. Given the stability of the \GSD{} error across different parameters, we next visualized the effect of these parameters on computational time in \cref{fig:ablation_runtime}.

The heatmap in \cref{fig:ablation_runtime} presents total computational time for each parameter combination, where darker hues indicate longer computational times. From this visualization, it is evident that the optimal population size is contingent on the dataset in question. For instance, while the \acscovdata{CA}, \acsempdata{CA}, and \acsmobdata{CA} datasets benefited from larger population sizes, the \acsincdata{CA} and \acstradata{CA} datasets performed better with smaller ones. However, we find that a robust parameter configuration that yields consistently good performance across all tasks was found to be $\Pmut=100$ and $\Pcross=100$.

\input{tables/ablation_summary}

\input{Figures/ablation_runtime}

\paragraph{Genetic Operators Rate}
In this section, we investigate an alternate approach to the \GSD{} algorithm that formulates new candidates by modifying multiple rows of synthetic data via either mutation or crossover operations. Typically, as detailed in \cref{sec:GSDdesc}, the \GSD{} method generates a new candidate synthetic data by adjusting a single entry of the optimal synthetic dataset, $\Dhat_g^\star$. However, by applying genetic operators to multiple rows of $\Dhat_g^\star$, we can potentially accelerate the convergence rate. We hence introduce `mutation rate' and `crossover rate' to denote the number of rows altered in each operation. For example, if the mutation rate is $k$, each candidate solution generated using the mutation strategy will differ from $\Dhat_g^\star$ at $k$ locations.

We conduct an experiment to understand the influence of these rates, in which we vary both the mutation rate and the crossover rate. We run \GSD{} on \acsmobdata{CA} and with all 2-way marginal statistics. The privacy parameter is set to $\epsilon=1$, the maximum number of generations capped at $100,000$, the population size constrained to $100$, and the size of the synthetic data set at $1,000$.

\cref{tab:parameters} depicts the final average error and the corresponding runtime in seconds for each combination of mutation rate and crossover rate. Our findings indicate that both the mutation rate and the crossover rate have a significant impact on error and run time. Interestingly, a mutation rate and crossover rate of $1$ yielded the best average error and run time across all combinations, establishing it as a consistently dependable choice in our experiment. This result emphasizes the importance of selecting these rates carefully for optimal algorithm performance. It further provides evidence that amplifying the level of perturbation in the candidate generation step can destabilize \GSD{} and lead to suboptimal performance.

\input{tables/mutation_cross}

%% file: Figures/prefix.tex
\begin{figure}[!htb]
    \centering
    \ificml
\else 
    \includegraphics[width=0.99\linewidth]{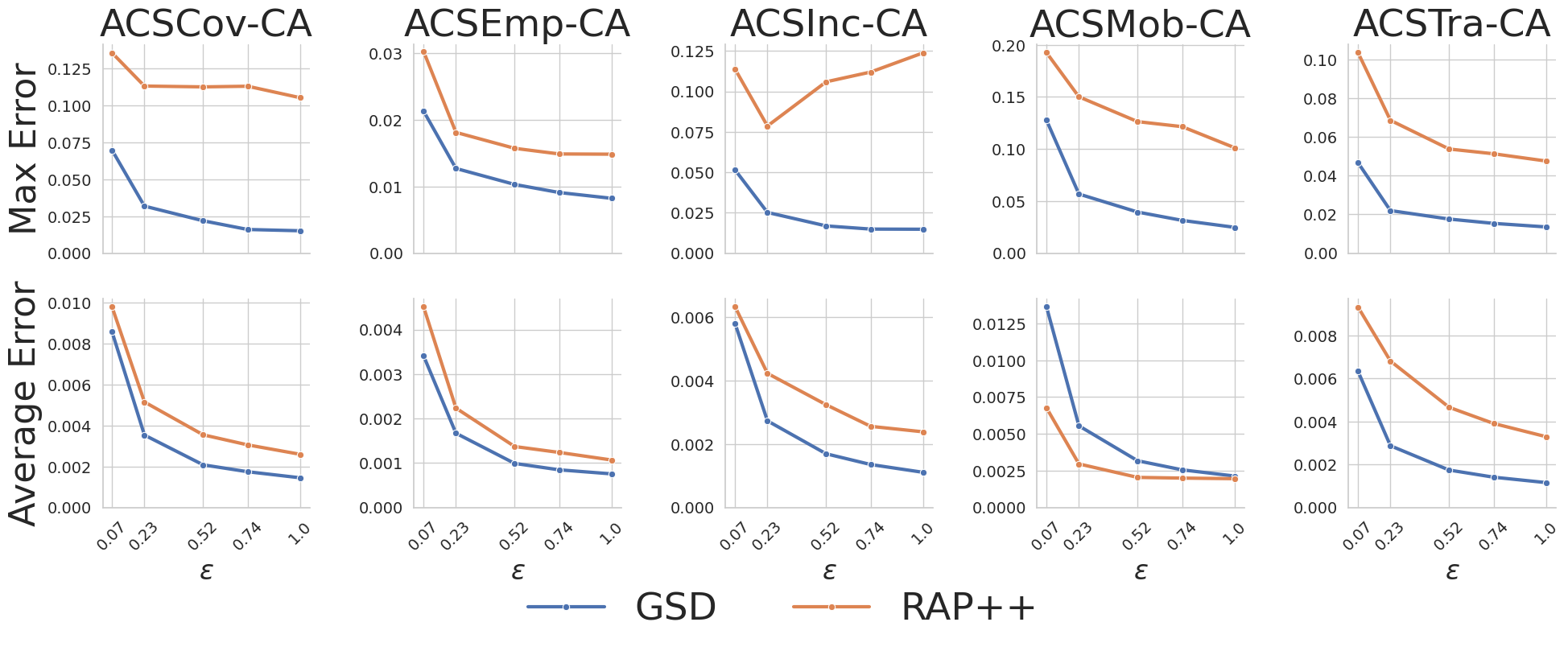}
    \fi
    \caption{\textbf{Prefix Queries:} Max and mean error evaluated on all ACS tasks with non-differentiable queries--- 200K random prefixes. 
    Using adaptive version of \GSD{} and \rappp{}. 
    }
    \label{fig:errors_prefix}
\end{figure}


%% file: Figures/halfspace.tex
\begin{figure}[!htb]
    \centering
    \ificml
\else 
    \includegraphics[width=1\linewidth]{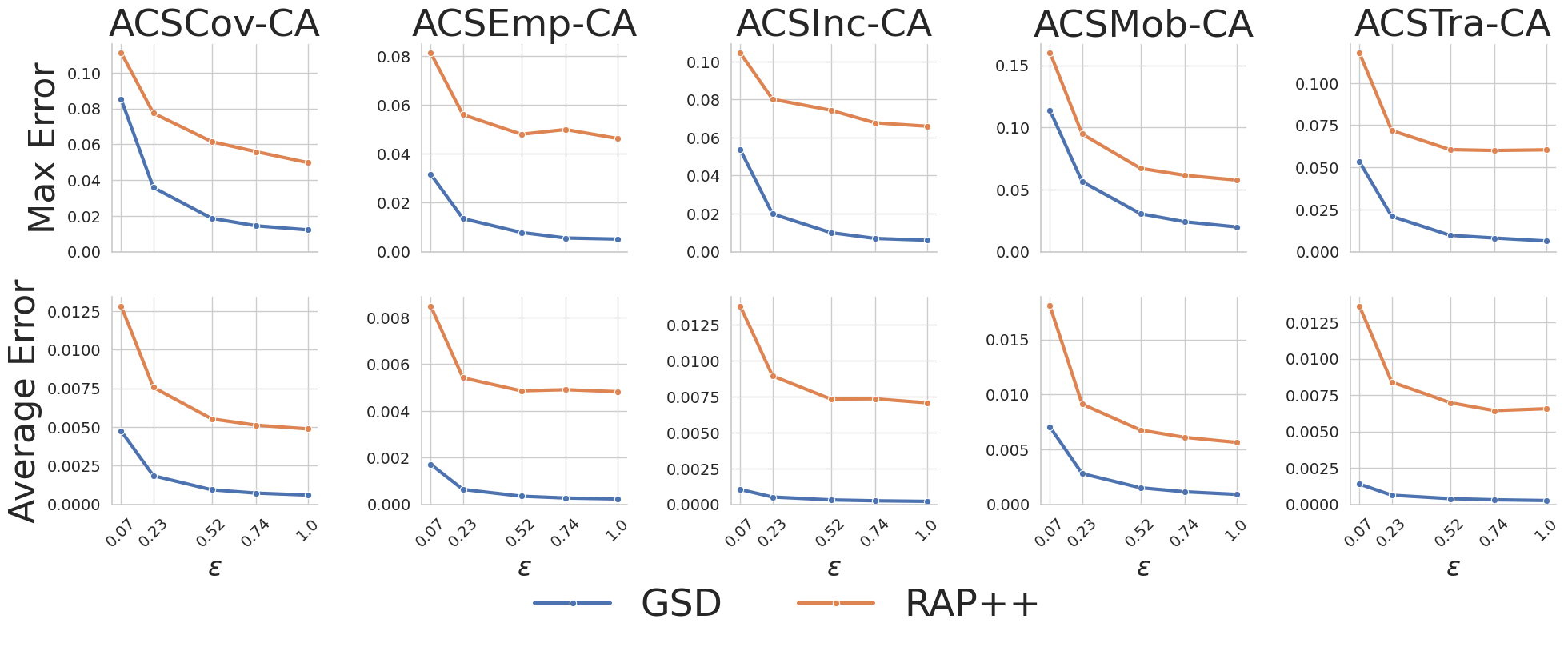}
    \fi
    \caption{\textbf{Halfspace Queries:} Max and average error evaluated on all ACS tasks with non-differentiable queries--- 200K random halfspaces. 
    Using adaptive version of \GSD{} and \rappp{}.
    }
    \label{fig:errors_halfspace}
\end{figure}

%% file: Figures/oneshot_categorical.tex
\begin{figure}[!htb]
    \centering
    \ificml
\else
    \includegraphics[width=1\linewidth]{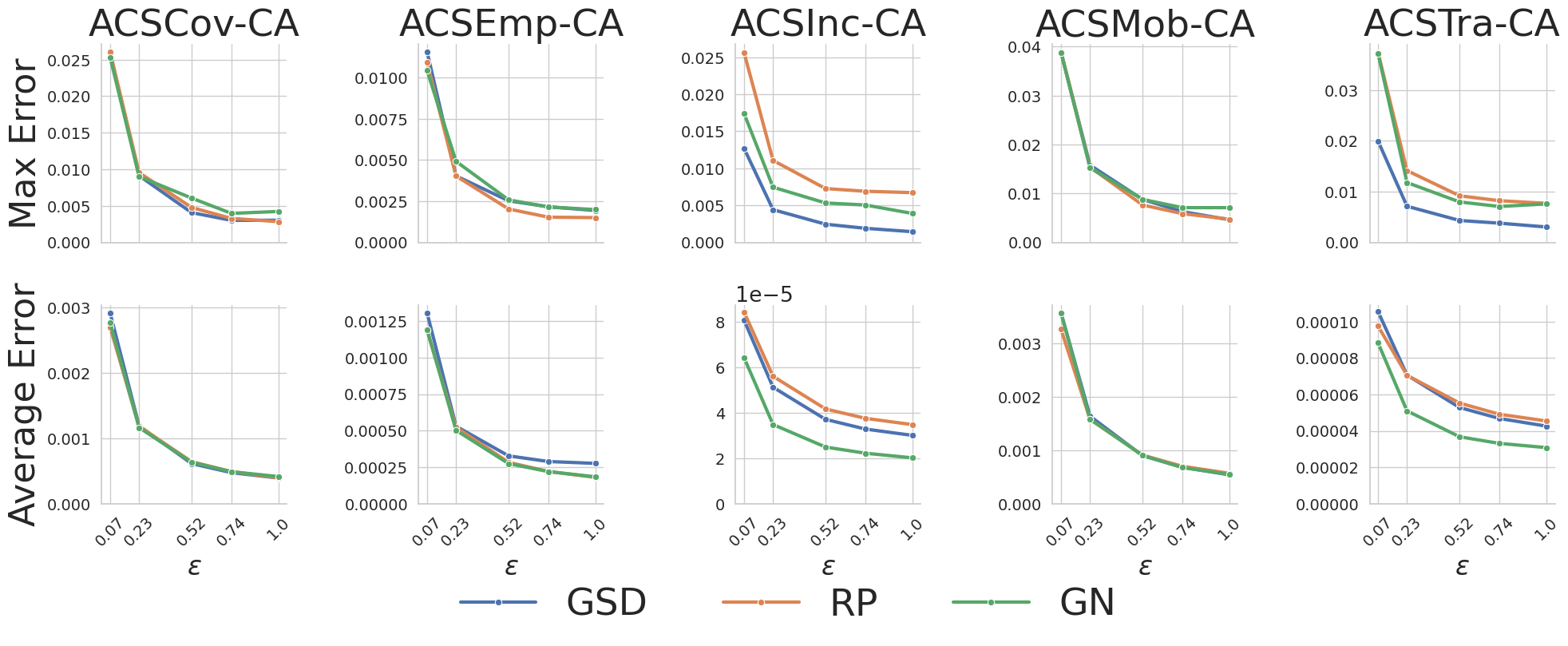}
    \fi
    \caption{
    \textbf{2-way Categorical Marginals: }
    For this experiment, we ran the one-shot version of \privGA{} and \rap{}. The plot shows  the max and average error evaluated on all ACS tasks with differentiable queries--- 2-way categorical marginal queries.
    }
    \label{fig:oneshot_2way_cat}
\end{figure}

%% file: Figures/oneshot_ranges.tex
\begin{figure}[!htb]
    \centering
    \ificml
\else
    \includegraphics[width=1\linewidth]{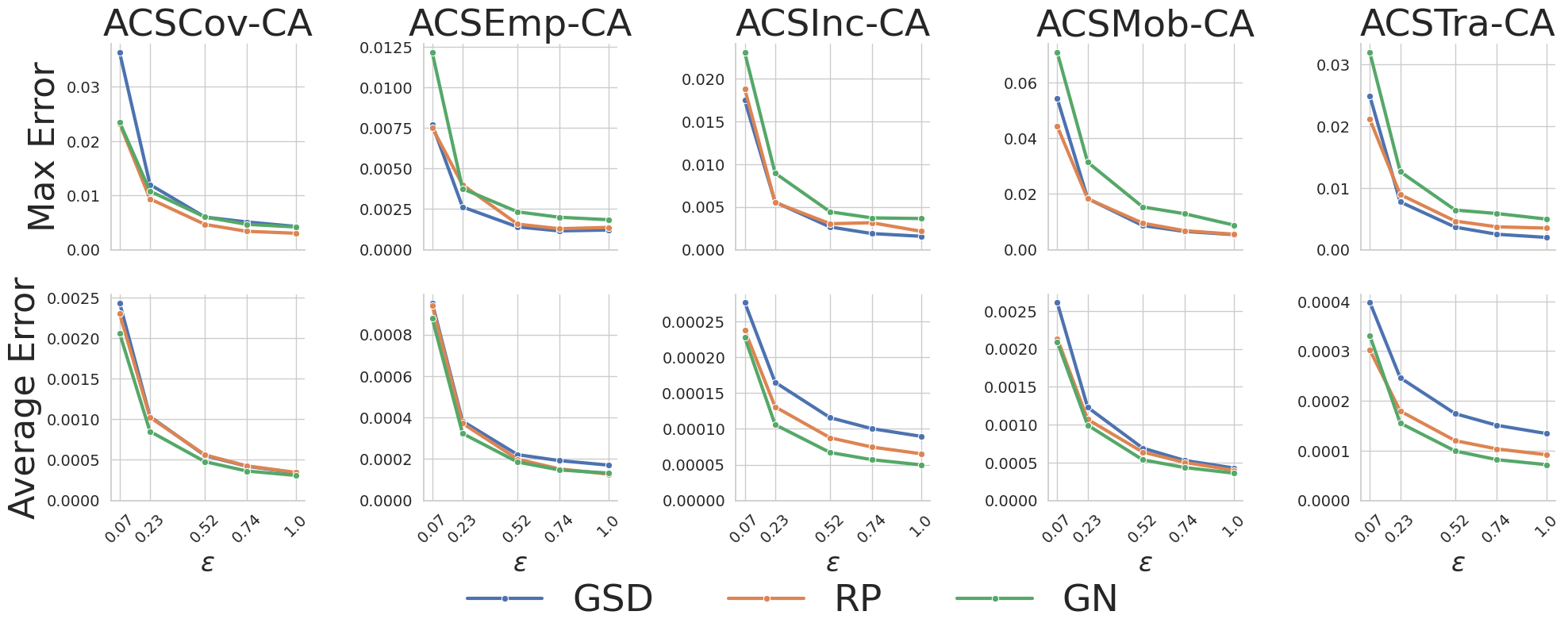}
    \fi
    \caption{
    \textbf{2-way Binary Tree Marginals:}
    For this experiment, we ran the one-shot version of \privGA{} and \rap{}. The plot shows  the max and average error evaluated on all ACS tasks with 2-way \emph{range} marginal queries.
    }
    \label{fig:oneshot_2way_bt}
\end{figure}

%% file: Figures/adaptive_3way_categorical.tex
\begin{figure}[!htb]
    \centering
    \ificml
\else  
    \includegraphics[width=1\linewidth]{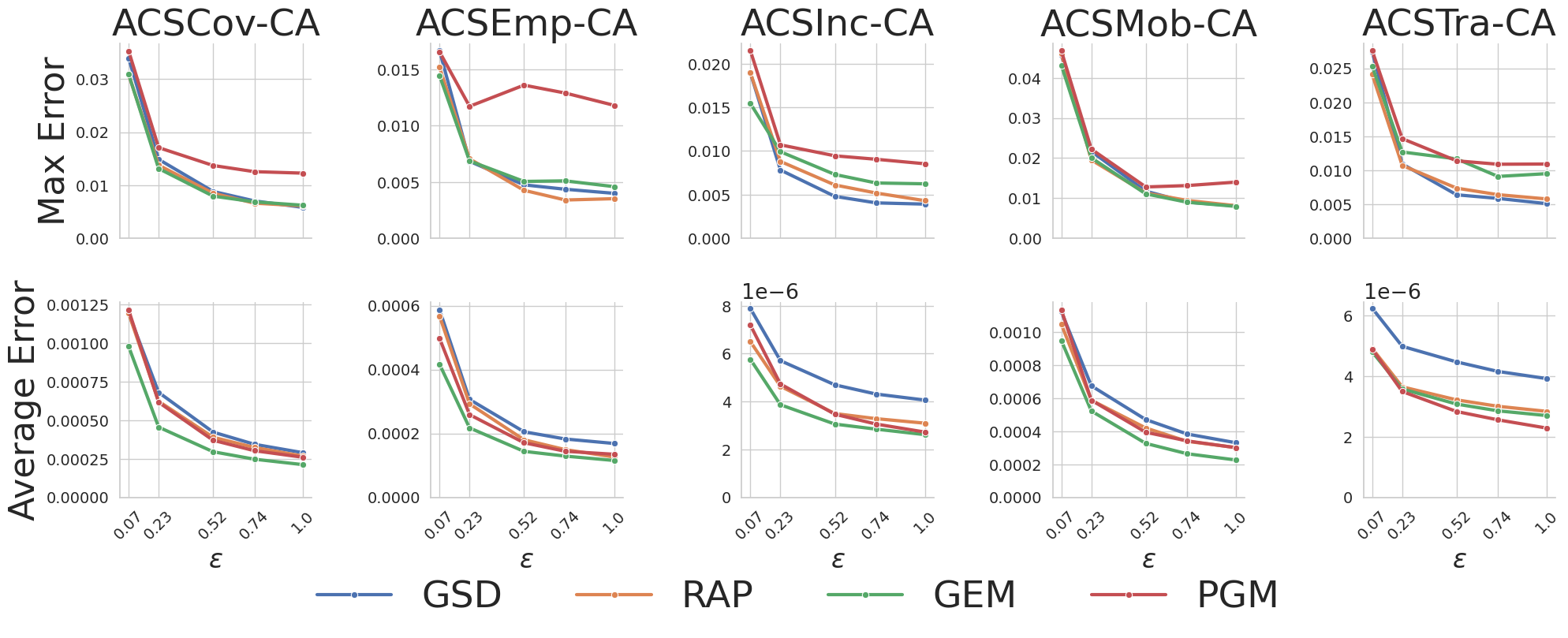}
    \fi
    \caption{
    \textbf{Adaptive $3$-way Categorical Marginals:}
    For this experiment, we ran the \emph{adaptive} version of \privGA{}, \rap, \gem{} and \pgm. The plot shows the max error evaluated on all ACS tasks with 3-way \emph{categorical} marginal queries. 
    }
    \label{fig:ada_3way_cat}
\end{figure}

%% file: tables/ablation_summary.tex
\begin{table}[ht]
    \centering
    \begin{tabular}{c | c | c c c c }
    \toprule
Data & Queries & Min Average Error & Max Average Error & Min Run Time(s) & Max Run Time(s)\\
\midrule
  \acscovdata{CA}&15894&0.0005921&0.0006259&50&247\\
  \acsempdata{CA}&10570&0.0004293&0.0004939&41&206\\
  \acsmobdata{CA}&38077&0.0005560&0.0005838&73&335\\
  \acsincdata{CA}&274641&0.0000761&0.0000849&165&510\\
  \acstradata{CA}&438594&0.0000960&0.0001012&234&598\\
    \bottomrule
    \end{tabular}
    \caption{ \textbf{Ablation Study Summary:}  A summary of the \GSD{} ablation study outlined in \cref{sec:ablationstudy}. The study entails running \GSD{} to optimize over all $2$-way binary-tree marginals in a oneshot setting, testing various $\Pmut$ and $\Pcross$ parameter combinations. The `Queries' column shows the total size of the query set designated for each dataset task. The remaining columns indicate best and worst case error/runtime for each ($\Pmut$, $\Pcross$) pairing per dataset. This summary suggests that \GSD{} error is relatively insensitive to $\Pmut$ and $\Pcross$ parameters, but the choice of these parameters significantly affects runtime.
   }
    \label{tab:ablationsummary}.
\end{table}

%% file: Figures/ablation_runtime.tex
\begin{figure}
    \centering
    \includegraphics[width=\textwidth]{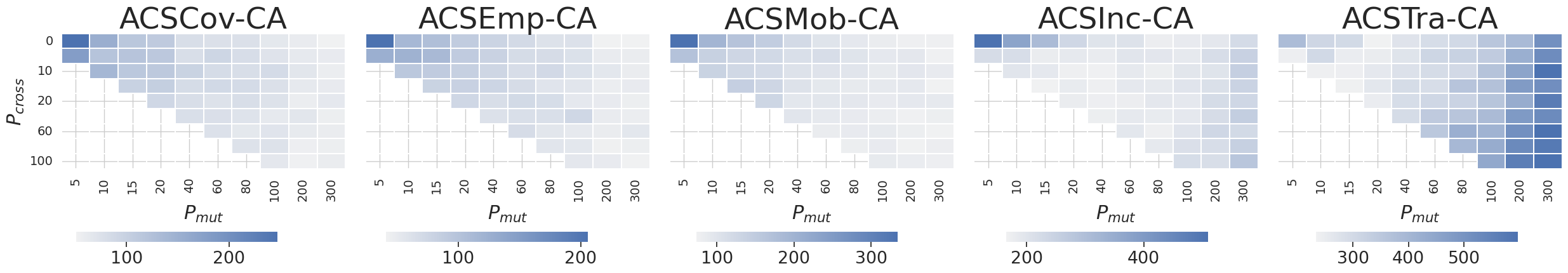}
    \caption{\textbf{ Ablation Study Runtimes:} 
    Heat maps displaying the total runtime of \GSD{} optimizing over all $2$-way binary-tree marginals in the oneshot setting. This is carried out for varying $\Pmut$ and $\Pcross$ parameters across five datasets. The constant parameters are: $G=300000$, $N'=1000$, $\epsilon=1$, and $\delta=1/N^2$. The \GSD{} algorithm incorporates the early stopping rule described in \cref{sec:GSDdesc}. Darker colors correspond to longer runtimes for each $\Pmut$/ $\Pcross$ pairing. A color-runtime legend is provided under each heat map. For a summary of error, runtime, and query set size in this ablation study, refer to \cref{tab:ablationsummary}.
    }

    \label{fig:ablation_runtime}
\end{figure}

%% file: tables/mutation_cross.tex
\begin{table}[h]
    \centering
    \begin{tabular}{c c| c c c c c}
    \toprule
    & & \multicolumn{5}{c}{Mutation Rate} \\ 
    &     & 1 & 2 & 5 & 10 & 25 \\
         \midrule  
  \multirow{5}{*}{ \rotatebox{90}{Crossover Rate}} &       1 & \textbf{0.0476 / (83)} & 0.0474 / (102) &0.0495 / (129) & 0.0553 / (162) &0.0826 / (266) \\ 
& 2 &
0.0480 / (117) &
0.0487 / (116)&
 0.0507 / (135)&
0.0590 / (168)&
0.0860 / (272) \\
& 5 &
0.0542 / (127) &
0.0562 / (135) &
0.0617 / (154) &
0.0728 / (189) &
0.0995 / (293) \\
& 10 & 
0.0825 / (161) &
0.0839 / (168) &
0.0908 / (189) &
 0.0987 / (223) &
0.1216 / (327) \\
& 25 & 
0.1922 / (264) &
0.1942 / (272) &
0.1918 / (291) &
0.1905 / (326) &
0.1945 / (435)
    \end{tabular}
    \caption{
    Effect of `mutation rate' and `crossover rate'. The numbers represent the average error for answering 2-way range queries and, in parenthesis, it is the total runtime in seconds.
    Results are averaged over 3 runs. For all experiments we fixed the following parameters: Input data is \acsmobdata{CA}, the max number of generations was $100,000$, the synthetic data size was $1,000$, $\Pmut=50$ and $\Pcross=50$, and the privacy parameter $\epsilon=1$. }
    \label{tab:parameters}
\end{table}

%% file: Arxiv/docs/conclusion.tex
\section{Conclusion}

We present \privGA{}, a versatile genetic algorithm that can generate synthetic data to approximate a wide range of statistical queries, regardless of their differentiability. \privGA{} uses novel crossover and mutation routines that are crucial to its success. Our empirical evaluation of the \privGA{} mechanism on mixed-type data demonstrates that \privGA{} outperforms the state-of-the-art method, \rappp{}, on non-differentiable queries and matches first-order optimization methods for differentiable ones. Overall, genetic algorithms present a promising solution for generating mixed-type synthetic data, and we hope our work will inspire further exploration of such approaches in privacy research. The \GSD{} source code is publicly available at \url{https://github.com/giusevtr/private_gsd}.

%% file: Arxiv/docs/accuracy.tex
\section{Accuracy Analysis of \privGA{}}

Our mechanism inherits similar oracle accuracy guarantees of the (non-adaptive) \rap{} mechanism, which were proven in \citet{aydore2021differentially}. 

\begin{theorem}
\label{thm:accurayAppendix}
For a discrete data domain $\cX$ and any dataset $D\in\cX^N$ with $N$ rows.
Fix privacy parameters  $\epsilon, \delta>0$, the synthetic dataset size $N'$, and any set of $m$ queries $Q:\cX^*\rightarrow[0,1]^m$. If the one-shot \privGA{} mechanism solves the minimization in the projection step exactly, then 
the one-shot \privGA{} mechanism outputs a synthetic data $\Dhat$ with
 average error bounded as
\begin{align*}
  \sqrt{\tfrac{1}{m} \| Q(D) - Q(\Dhat) \|_2^2} \leq 
   O\pp{
   \frac{
    \pp{(\log(|\cX|/\beta)\ln(1/\delta)}^{1/4}
   }{\sqrt{\epsilon\cdot N }}
   + \frac{\sqrt{\log(m)}}{\sqrt{N'}}
   }
\end{align*}
\end{theorem}
\begin{proof}

We begin with the projection mechanism, which has the following guarantee proven by \citet{NKL}: Given the discrete data domain $\cX$, let $\cX^*$ be the set of datasets of arbitrary size with rows from the domain $\cX$.

The projection mechanism by \citet{NKL}, perturbs the answers of $Q$ on dataset $D$ using the Gaussian mechanism to obtain a vector $\hat{a}$, and then finds a synthetic dataset $\Dhat^*\in \cX^*$ by solving the following objective
\begin{align*}
    \Dhat^* \leftarrow \argmin_{\Dhat \in \cX^*} \|\hat{a} - Q(\Dhat)\|
\end{align*}
with the following accuracy guarantee:
\begin{align*}
    \sqrt{ \tfrac{1}{m}\| Q(D) - Q(\Dhat^*) \|_2^2} \leq \alpha  = O\pp{\frac{\pp{\ln(|\cX| /\beta)\ln(1/\delta)}^{1/4}}{\sqrt{\epsilon\cdot N}}} 
\end{align*}

Next, let $m$ be the number of queries in $Q$. By a simple sampling argument, proven in Lemma 3.7 of \citet{BLR08}, there exists a dataset $\Dhat$ of size $N'$ that satisfies the following:
\begin{align*}
     \|Q(\Dhat) - Q(\Dhat^*)\|_{\infty} \leq \sqrt{\frac{\log(m)}{N'}}
\end{align*}

Finally, suppose that the output of \privGA{} is $\Dhat\in\cX^{N'}$, which is generated by solving an optimization problem exactly as follows:
\begin{align*}
    \Dhat \leftarrow \argmin_{\Dhat \in \cX^{N'}} \|\hat{a} - Q(\Dhat)\|
\end{align*}

Then, by triangle inequality:
\begin{align*}
    \sqrt{ \tfrac{1}{m}\| Q(D) - Q(\Dhat) \|_2^2}
    &= \sqrt{ \tfrac{1}{m}\| Q(D) - Q(\Dhat^*) + Q(\Dhat^*) - Q(\Dhat) \|_2^2} \\
    &\leq  \sqrt{ \tfrac{1}{m}\| Q(D) - Q(\Dhat^*) \|_2^2} + \sqrt{ \tfrac{1}{m}\|  Q(\Dhat^*) - Q(\Dhat) \|_2^2} && (\text{Triangle inequality})\\
    &\leq \sqrt{ \tfrac{1}{m}\| Q(D) - Q(\Dhat^*) \|_2^2} + \|Q(\Dhat) - Q(\Dhat^*)\|_{\infty} \\
    &\leq O\pp{\frac{\pp{\ln(|\cX| /\beta)\ln(1/\delta)}^{1/4}}{\sqrt{\epsilon\cdot N}}}  + \sqrt{\frac{\log(m)}{N'}} \\
\end{align*}
\end{proof}

Despite the bounds in \cref{thm:accurayAppendix} matching from the bounds in \citet{aydore2021differentially}, there is no guarantee about the performance \privGA{}'s optimization. Therefore, we empirically show that \privGA{} has comparable performance to \rap{} on matching marginal statistics over many different datasets.
In \cref{fig:oneshot_2way_cat}, we provide results on categorical marginals and with different ACS tasks as defined in Table \ref{tab:acs_data_attributes}. This brief experiment also shows that, like in the adaptive setting, \privGA{} achieves similar performance to \rap{} on differentiable queries in the one-shot setting across many datasets. We use the same hyperparameters used in our adaptive experiments (Table \ref{tab:hyperparameters}), excluding our choices for $T$, which are not applicable in the one-shot setting.

%% file: Arxiv/docs/algorithm_details.tex
\section{Additional algorithm details}

\subsection{Adaptive Selection}\label{appx:adaptive}

We present the Adaptive Selection framework in Algorithm \ref{alg:adaptive}, and also provide its privacy proof below.

\input{icml/algorithms/adaptive_framework.tex}


\subsection{\rappp{} and its limitations}\label{appx:rappp}
\input{icml/docs/rap++_v2.tex}

\subsection{Modifications to \rp{} / \gn{}}\label{appx:rapgem}

Suppose we have a data domain $\gX = \gX_1 \times \ldots \times \gX_d$ contains $d$ discrete attributes. Then both \rp{} and \gn{} (both the one-shot and adaptive variants) model such features as a mixture of $K$ product distributions over attributes $\gX_i$ (and possible values they can take on). We denote each product distribution as $P_i = \prod_{j=1}^d P_{ij}$, where $P_{ij}$ is a distribution over the discrete values of $\gX_j$. The difference between \rp{} and \gn{} then is how each distribution $P_{ij}$ is parametrized. In \rp, the parameters $\theta$ are exactly the values $P_{ij}$. Meanwhile, \gn{} instead outputs the values $P_{ij}$ using a neural network $F_{\theta}$, which takes in as input Gaussian noise $\mathbf{z} \sim \mathcal{N}\left(0, I_K \right)$.

In either case, query answers for $k$-way categorical queries can be calculated as a product over various columns in $P_i$. For example (relaxing notation), suppose $P_{i, \textrm{COLOR}=\textrm{red}}$ and $P_{i, \textrm{SHAPE}=\textrm{circle}}$ correspond to the probabilities that attributes $\textrm{COLOR}=\textrm{red}$ and $\textrm{SHAPE}=\textrm{circle}$. Then the 2-way marginal query counting the proportion of samples that are red circles can be written down as the product query $\frac{1}{K} \sum_{i=1}^{K} P_{i, \textrm{COLOR}=\textrm{red}} \times P_{i, \textrm{SHAPE}=\textrm{circle}}$.

Now suppose we have some numerical column that has been discretized into bins. Then a $k$-way range marginal query can be simple written down as a sum of product queries. For example, suppose we have an attribute QUANTITY that has been discretized into bins of size 10. Then the $2$-way range query for proportion of samples that are red and have quantity between 1 and 20 can be written down as $\frac{1}{K} \sum_{i=1}^{K} P_{i, \textrm{COLOR}=\textrm{red}} \times \left[ P_{i, \textrm{QTY}=\textrm{1-10}} + P_{i, \textrm{QTY}=\textrm{11-20}} \right]$. Therefore, we are able to still optimize over such queries, since they are simply sums over (differentiable) product queries that the algorithms were originally designed for.

Finally, while \citet{liu2021iterative} and \citet{aydore2021differentially} propose different optimization strategies for their algorithms at each round, we introduce a simpler update procedure that we find performs well across our experiments. We present this procedure in Algorithm \ref{alg:rapgem_update}.

\begin{algorithm}[!tbh]
\caption{\rp{} / \gn{} Update}
\label{alg:rapgem_update}
\begin{algorithmic}
\STATE {\bfseries Input:} $\rp{}/\gn{}_\theta$, queries (sampled via the exponential mechanism) $Q = \anglebrack{q_1, \ldots, q_t}$, noisy answers (measured via the Gaussian mechanism) $A = \anglebrack{a_1, \ldots, a_t}$, max iterations $M$, batch size $B$ \\

\FOR{$m = 1 $ {\bfseries to} $2M$}
    \STATE Let synthetic answers $\hat{A} = \rap{}/\gem{}_\theta(Q)$
    \STATE Let errors $P = |A - \hat{A}|$
    \IF{${m \bmod 2} = 0$}
        \STATE Sample a set of query indices $S$ of size $B$ uniformly
    \ELSE
        \STATE Sample a set of query indices $S$ of size $B$ proportional to errors $P$
    \ENDIF
    \STATE Update $\theta$ via stochastic gradient decent according to $\norm{P_S}_1$
\ENDFOR
\end{algorithmic}
\end{algorithm}

%% file: icml/algorithms/adaptive_framework.tex
\begin{algorithm}[!hbt]
\caption{Adaptive Framework for Synthetic Data  }
 \label{alg:adaptive}
\begin{algorithmic}[1]
\STATE\textbf{Input:} A private dataset $D\in\cX^N$ with $N$ rows and $d$ columns, A set of queries $Q$, number of adaptive epochs $T$, number of samples per epoch $S$, privacy parameter $\rho$.

\STATE Split privacy budget $\rho' = \rho / (2 \cdot    T \cdot S)$.

\STATE Randomly initialize a synthetic data $\Dhat_{1}$.
\FOR{$t\in [T]$}{
    \FOR{$i \in [S]$}
        \STATE \textbf{Select} a query $q_{t,i}$ from $Q$ using a $\rho'$-zCDP select mechanism and score function 
        $$\text{score}(q) = \| q(D) - q(\Dhat_t) \|_{\infty} $$
        \STATE \textbf{Measure} statistic with Gaussian mechanism: 
$$\hat{a}_{t,i} \leftarrow Q(D) + \cN\pp{0, \tfrac{\Delta(q)}{2 \cdot \rho'}}$$
    \ENDFOR
    \STATE Let $\hat{a}_t = \left( a_{1,1}, \ldots, a_{t, S} \right)$  and $Q_t(\Dhat) = (q_{1, 1}(\Dhat), \ldots, q_{t, S}(\Dhat ) ) $.
    \STATE \textbf{Project} $\widehat{D}_{t+1} \leftarrow \argmin_{\widehat{D}} \|\hat{a}_t - Q_t({\Dhat})\|_2$ 
}
\ENDFOR

\STATE \textbf{return} $\Dhat_{T+1}$
\end{algorithmic}
\end{algorithm}

%% file: icml/docs/rap++_v2.tex
\input{icml/algorithms/rap++.tex}
We present the details of \rappp{ }'s first-order optimization routine in  \cref{alg:rappp}. As mentioned before, this technique relaxes the objective to gain differentiability. However, some issues arise.
In order to showcase the issue with the relaxation step, we construct an example, where \rappp{} fails to optimize even a single prefix query on a dataset with one real-valued feature. Let $f_{0.5}(x) = \mathbb{I}\{x \leq 0.5\}$ be a prefix function with threshold $0.5$,  then for a one-dimensional dataset $D\in [0,1]^N$ with $N$ rows the corresponding statistical query is $q_{0.5}(D) = \frac{1}{N}\sum_{x\in D} f_{0.5}(x)$. 
And, let $\hat{a} = q_{0.5}(D) $ be the true answer of the prefix query on dataset $D$.
Then,  \rappp{} wants to find a synthetic data $\Dhat$ that minimizes the objective $\|\hat{a} - q_{0.5}(\Dhat)\|_2$, which is non-differentiable due to the threshold query. Therefore, \rappp{} optimizes a differentiable relaxation of the objective defined by $q_{0.5}$. The relaxation stop replaces the prefix function $f_{0.5}$, with a \emph{Sigmoid} function, which has well-behaved gradients with adjustable magnitudes via the inverse temperature parameter.
 A { Sigmoid} function   with threshold $0.5$ and {\em inverse temperature parameter} $\invtemp$ is given by
$    f_{0.5}^{[\invtemp]}(x)=\tfrac{1}{\pp{{1+\exp(-\invtemp\cdot (x-0.5))}} }$

The corresponding Sigmoid approximation to the prefix query is given by $\tilde{q}_{0.5}^\invtemp(D) = \frac{1}{N}\sum_{x\in D} f_{0.5}^{[\invtemp]}(x)$.
The relaxed objective then becomes $\|\hat{a} - \tilde{q}_{0.5}^\invtemp(\Dhat)\|_2$.

Now, using a simple example, we show that \rappp{} can get stuck in a bad local minimum of the differentiable loss function. Suppose that the input data $D$ consist of $N/2$ entries with value equal to $0$ and $N/2$ entries with value equal to $1$, that is $D=\{0, \ldots, 0, 1, \ldots, 1\}\in [0,1]^N$. Then the input dataset $D$, evaluates to $q_{0.5}(D) = 0.5$ on the prefix query (half the data points are bellow the $0.5$ threshold).
Next, consider the synthetic data point $\Dhat=\{0.5, \ldots, 0.5\}$, which evaluates to $\tilde{q}_{0.5}(\Dhat) = 0.5$  on the Sigmoid prefix function, then,  $\Dhat$ is a local minimum of the objective function ($\|\hat{a} - \tilde{q}_{0.5}^\invtemp(\Dhat)\|_2 = 0$).
%
%
However, when evaluated on the actual prefix query, the synthetic data point $\Dhat$ evaluates to $q_{0.5}(\Dhat) = 1$. Therefore, the error on the original loss function is $0.5$.

This simple example shows how the first-order optimization procedure by \citet{vietri2022private} can get stuck in a bad local minimum, resulting in a poor approximation of the statistical queries. Furthermore, in the experiment section, we will show that the \rappp{} algorithm fails in practice.

%% file: icml/algorithms/rap++.tex
\newcommand{\qtil}{\tilde{q}}
\begin{algorithm}[t]
\small
\begin{algorithmic}[1]
\caption{Relaxed Projection with Sigmoid Temperature Annealing}\label{alg:rappp}
\STATE \textbf{Input: } A set of $m$ sigmoid differentiable queries $\{\tilde{q}^{[\cdot]}_i\}_{i\in [m]}$, a set of $m$ target answers $\hat{a}=\{\hat{a}_i \}_{i\in [m]}$, initial inverse temperature $\sigma_1\in\mathbb{R}^+$, stopping condition $\gamma>0$, and initial dataset $\widehat{D}_1$.



\FOR{$j=1$ \textbf{ to } $J$}
\STATE Set inverse temperature $\sigma_j = \sigma_1 \cdot 2^{j-1}$.
\STATE Define the sigmoid differentiable loss function: 
$    \Loss_j(\widehat{D}) = \sum_{i \in [m]} \left( \qtil^{[\sigma_j]}_i(\widehat{D}) - \hat{a}_i \right)^2$
\STATE  Starting with $\widehat{D}\leftarrow \widehat{D}_j$.
Run  gradient descent on $L_j({\widehat{D}})$ until    $\|\nabla \Loss_j(\widehat{D})\| \leq \gamma$.
Set $\widehat{D}_{j+1} \leftarrow \widehat{D}$.
\ENDFOR
\STATE \textbf{Output} $\widehat{D}_{J+1}$
\end{algorithmic}
\end{algorithm}

%% file: Arxiv/docs/experiment_details.tex
\section{Additional experiment details}

In this section we include additional details about our experiments, including information about datasets, queries, and hyperparameters.

\paragraph{Data.} In Table \ref{tab:acs_data_attributes}, we list all attributes for each ACS task that are used in empirical results presented in either the main paper or this appendix.

\begin{table}[!htb]

\centering

\begin{tabular}{l | l| l}
\toprule
Dataset & Categorical Attributes & Numeric Attributes \\
\midrule
\multirow{2}{*}{\acsmob} & MIL, DREM, CIT, DIS, COW, MAR, SCHL, MIG, & JWMNP, PINCP \\
&  NATIVITY, ANC, DEAR, DEYE, SEX, ESR, GCL & WKHP, AGEP \\
&  RAC1P, RELP & \\
\midrule
\multirow{1}{*}{\acsemp} & COW, OCCP, SCHL, MAR, SEX, PINCP, RAC1P  & AGEP \\
& RELP, POBP & \\
\midrule
\multirow{2}{*}{\acscov} & MIL, FER, DREM, CIT, DIS, MAR, SCHL, PUBCOV & \multirow{2}{*}{PINCP, AGEP} \\
& MIG, NATIVITY, ANC, DEYE, SEX, ESR, RAC1P & \\
&  DEAR, ESP& \\ 
\midrule
\multirow{1}{*}{\acsinc} & COW, OCCP, SCHL, MAR, SEX, PINCP, RAC1P & WKHP, AGEP \\
&RELP, POBP & \\ 
\midrule
\multirow{2}{*}{\acstra} & JWMNP, PUMA, CIT, DIS, OCCP, SCHL, JWTR, MAR & \multirow{2}{*}{AGEP, POVPIP} \\
& MIG, SEX, RAC1P, POWPUMA, RELP, ESP & \\
\midrule
\multirow{4}{*}{\acsmultitask} & OCCP, MIG, CIT, ESP, ANC, PUMA, FER, DEAR, DEYE & INTP, WAGP \\
&  ESR, FOCCP, JWTR, WAOB, JWMNP\_bin, DREM, GCL & SEMP, POVPIP \\ & MIL, SCHL, RELP, MAR, NATIVITY, POWPUMA & WKHP, JWRIP \\ & RAC1P, DIS, COW, PUBCOV, SEX, PINCP & AGEP \\

\bottomrule
\end{tabular}
\caption{We list the ACS data attributes used in our experiments.}
\label{tab:acs_data_attributes}
\end{table}

\begin{table}[!htb]
\centering
\begin{tabular}{l | c c c c c }
    \toprule
    & mobility & coverage & employment & income & travel \\
    \midrule
    $2$-way Cat. Marginals & 136 & 136 & 120 & 36 & 91 \\
    $3$-way Cat. Marginals & 680 & 680 & 560 & 84 & 364 \\
    $2$-way BT Marginals & 68 & 34 & 16 & 18 & 28 \\
    \bottomrule
\end{tabular}
\caption{Number of workloads for categorical and range marginal queries used in our experiments.}
\label{tab:workloads}
\end{table}

\paragraph{Hyperparameters.} In Table \ref{tab:hyperparameters}, we list the hyperparameters used for our algorithms.

\input{tables/hyperparameters}

\paragraph{Additional results.} In addition, we include the accuracy of the logistic regression models trained on synthetic data generated from each method \cref{fig:MLaccuracy}. We observe that these results mirror the macro F1 scores presented in \cref{fig:MLf1}.

\begin{figure}[!htb]
    \centering
    \includegraphics[width=0.7\linewidth]{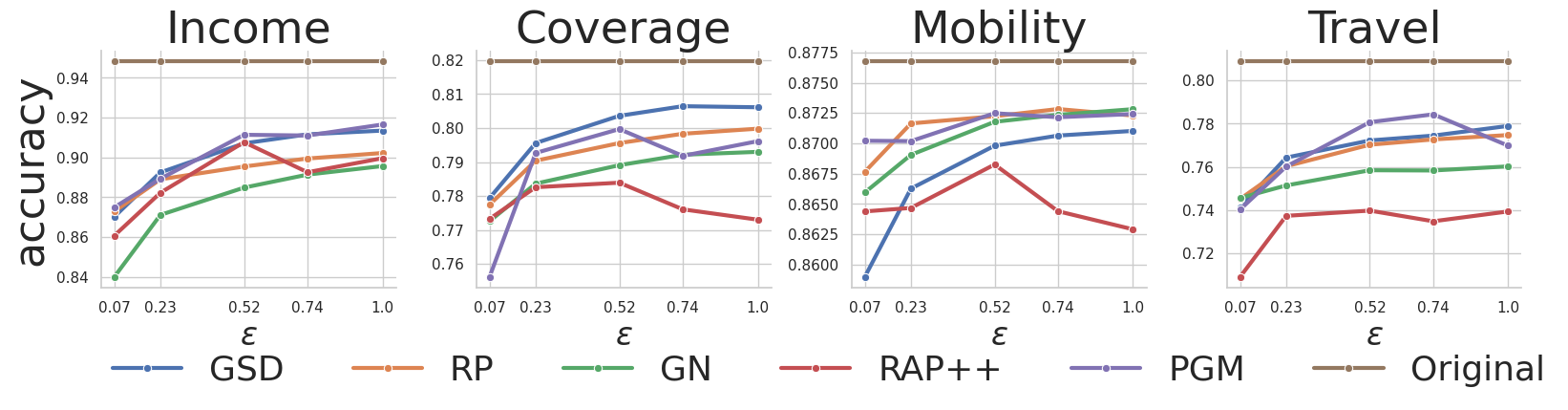}
    \caption{
    \textbf{Machine Learning Evaluation: } Accuracy test score of logistic regression trained on synthetic data. Each ACS dataset is partitioned into training and testing set. The train set is used to train the synthetic data algorithm, and the accuracy score is computed using the test set.
    Each algorithm is trained either on Binary Tree (BT) marginal queries or Halfspace (HS) queries.   We fix $\delta={1/N^2}$ and vary $\varepsilon\in\{0.07, 0.23, 0.52, 0.74, 1.0\}$. Results are averaged over 3 runs. }
    \label{fig:MLaccuracy}
\end{figure}




%% file: tables/hyperparameters.tex
\begin{table}[!htb]
\centering
\renewcommand{\arraystretch}{1.3}
\begin{tabular}{l l c c}
    \toprule
    Query Class & Method & Hyperparameter & Values \\
    \midrule
    \multirow{20}{*}{All Experiments} & 
    \multirow{5}{*}{\privGA{}}
    & \multirow{1}{*}{Data Size} & $2000$ \\
    & & \multirow{1}{*}{$\Pmut$} & $500/50$ \\
    & & \multirow{1}{*}{$\Pcross$} & $500/50$ \\
    & & \multirow{1}{*}{Elite Size} & $2$ \\
    & & \multirow{1}{*}{Max Generations} & $200000$ \\
    \cmidrule(lr){2-4}
    & \multirow{5}{*}{\rappp}
    & \multirow{1}{*}{Data Size ($N'$)} & $1000$ \\
    & & \multirow{1}{*}{Queries Sampled ($K$)} & $1000$ \\
    & & \multirow{1}{*}{Learning Rate} & $0.0005$, $0.001$, $0.002$, $0.005$, $0.01$ \\
    & & \multirow{2}{*}{Inverse Temp. ($\invtemp$)}
    & $2$, $4$, $8$, $16$, $32$, $64$, $128$ \\
    & & & $256$, $512$, $1024$, $2048$ \\
    \cmidrule(lr){2-4}
    & \multirow{4}{*}{\rap{} / \gem{}}
    & \multirow{1}{*}{\# Product Mixtures ($K$)} & $1000$ \\
    & & \multirow{1}{*}{Batch Size ($B$)} & $5000$ \\
    & & \multirow{1}{*}{Max Iterations ($M$)} & $1000$ \\
    & & \multirow{1}{*}{\# Samples} & $50000$ \\
    \cmidrule(lr){2-4}
    & \multirow{1}{*}{\rap} & \multirow{1}{*}{LR} & $0.03$ \\
    \cmidrule(lr){2-4}
    & \multirow{3}{*}{\gem{}}
    & \multirow{1}{*}{LR} & $0.0003$ \\
    & & \multirow{1}{*}{Hidden Layers} & $(128, 256)$ \\
    & & \multirow{1}{*}{Noise Dimension} & $16$ \\
    \cmidrule(lr){2-4}
    & \multirow{2}{*}{\pgmem}
    & \multirow{1}{*}{Max Iterations} & $250$ \\
    & & \multirow{1}{*}{\# Samples} & $N$ \\
    \midrule
    \multirow{2}{*}{Halfspace / Prefix} & 
    \multirow{1}{*}{\privGA{}} & \multirow{1}{*}{$T$} & $25$, $50$, $75$, $100$ \\
    \cmidrule(lr){2-4}
    & \multirow{1}{*}{\rappp} & \multirow{1}{*}{$T$} & $3$, $4$, $5$, $6$, $7$, $8$, $9$ \\
    \midrule
    \multirow{1}{*}{3-way Categorical} & 
    \multirow{1}{*}{All algorithms} & \multirow{1}{*}{$T$} & $25$, $50$, $75$, $100$ \\
    \bottomrule
\end{tabular}
\caption{Hyperparameters experiments (with adaptivity).}
\label{tab:hyperparameters}
\end{table}